\documentclass[a4paper,11pt]{amsart}
\usepackage{lmodern}  

\usepackage[T1]{fontenc}    % use 8-bit T1 fonts
\usepackage[top=1in, bottom=1in, left=1in, right=1in]{geometry}
\usepackage{hyperref}       % hyperlinks
\usepackage{url}            % simple URL typesetting
\usepackage{booktabs}       % professional-quality tables
\usepackage{amsmath,amssymb,amsfonts,amsthm,graphicx}       % blackboard math symbols
\usepackage{nicefrac}       % compact symbols for 1/2, etc.
\usepackage{microtype}      % microtypography

\usepackage{verbatim}
\usepackage{float}
\usepackage{enumerate}
\usepackage{enumitem}
\usepackage{wrapfig}

\usepackage{algorithm}
\usepackage{algpseudocode}
\usepackage{multirow, array}

\title{Query Complexity of Clustering with Side Information}

\author{Arya Mazumdar}
\author{Barna Saha}\thanks{ The authors are with College of Information and Computer Sciences, University of Massachusetts Amherst, emails: \url{arya@cs.umass.edu},\url{barna@cs.umass.edu}.This work is partially supported by NSF awards CCF 1642658,  CCF 1642550,  CCF 1464310, a Yahoo ACE Award and a Google Faculty Research Award.}

\makeatletter
\renewcommand\subsubsection{\@startsection{subsubsection}{2}%
  \z@{.5\linespacing\@plus.7\linespacing}{-.5em}%
  {\normalfont\bfseries}}
\makeatother

\usepackage{color}

\usepackage{hyperref}
\usepackage{caption}
\usepackage{microtype}

\DeclareGraphicsExtensions{.eps}

\newcommand\nc\newcommand
\nc\bfa{{\boldsymbol a}}\nc\bfA{{\boldsymbol A}}\nc\cA{{\mathcal A}}
\nc\bfb{{\boldsymbol b}}\nc\bfB{{\boldsymbol B}}\nc\cB{{\mathcal B}}
\nc\bfc{{\boldsymbol c}}\nc\bfC{{\boldsymbol C}}\nc\cC{{\mathcal C}}
\nc\sC{{\mathscr C}}
\nc\bfd{{\boldsymbol d}}\nc\bfD{{\boldsymbol D}}\nc\cD{{\mathcal D}}
\nc\bfe{{\boldsymbol e}}\nc\bfE{{\boldsymbol E}}\nc\cE{{\mathcal E}}
\nc\bff{{\boldsymbol f}}\nc\bfF{{\boldsymbol F}}\nc\cF{{\mathcal F}}
\nc\bfg{{\boldsymbol g}}\nc\bfG{{\boldsymbol G}}\nc\cG{{\mathcal G}}
\nc\bfh{{\boldsymbol h}}\nc\bfH{{\boldsymbol H}}\nc\cH{{\mathcal H}}
\nc\bfi{{\boldsymbol i}}\nc\bfI{{\boldsymbol I}}\nc\cI{{\mathcal I}}
\nc\bfj{{\boldsymbol j}}\nc\bfJ{{\boldsymbol J}}\nc\cJ{{\mathcal J}}
\nc\bfk{{\boldsymbol k}}\nc\bfK{{\boldsymbol K}}\nc\cK{{\mathcal K}}
\nc\bfl{{\boldsymbol l}}\nc\bfL{{\boldsymbol L}}\nc\cL{{\mathcal L}}
\nc\bfm{{\boldsymbol m}}\nc\bfM{{\boldsymbol M}}\nc\sM{{\mathscr M}}\nc\cM{{\mathcal M}}
\nc\bfn{{\boldsymbol n}}\nc\bfN{{\boldsymbol N}}\nc\cN{{\mathcal N}}
\nc\bfo{{\boldsymbol o}}\nc\bfO{{\boldsymbol O}}\nc\cO{{\mathcal O}}
\nc\bfp{{\boldsymbol p}}\nc\bfP{{\boldsymbol P}}\nc\cP{{\mathcal P}}
\nc\bfq{{\boldsymbol q}}\nc\bfQ{{\boldsymbol Q}}\nc\cQ{{\mathcal Q}}
\nc\bfr{{\boldsymbol r}}\nc\bfR{{\boldsymbol R}}\nc\cR{{\mathcal R}}
\nc\bfs{{\boldsymbol s}}\nc\bfS{{\boldsymbol S}}\nc\cS{{\mathcal S}}
\nc\bft{{\boldsymbol t}}\nc\bfT{{\boldsymbol T}}\nc\cT{{\mathcal T}}
\nc\bfu{{\boldsymbol u}}\nc\bfU{{\boldsymbol U}}\nc\cU{{\mathcal U}}
\nc\bfv{{\boldsymbol v}}\nc\bfV{{\boldsymbol V}}\nc\cV{{\mathcal V}}
\nc\bfw{{\boldsymbol w}}\nc\bfW{{\boldsymbol W}}\nc\cW{{\mathcal W}}
\nc\bfx{{\boldsymbol x}}\nc\bfX{{\boldsymbol X}}\nc\cX{{\mathcal X}}
\nc\bfy{{\boldsymbol y}}\nc\bfY{{\boldsymbol Y}}\nc\cY{{\mathcal Y}}
\nc\bfz{{\boldsymbol z}}\nc\bfZ{{\boldsymbol Z}}\nc\cZ{{\mathcal Z}}

\nc\diff{{\mathrm d}}
\nc\e{{\mathrm e}}
\nc\calC{{\mathcal C}}
\newcommand{\remove}[1]{}

\newcommand{\avg}{{\mathbb E}}

\newtheorem*{lemma*}{Lemma}
\newtheorem{corollary}{Corollary}
\newtheorem{theorem}{Theorem}
\newtheorem{lemma}{Lemma}
\theoremstyle{definition}

\theoremstyle{corollaryn}

\newtheorem{observation}[theorem]{Observation}
\newtheorem{problem}{Problem}

\newtheorem{remark}{Remark}

\newcommand{\cc}{{\sf Query-Cluster}}

\def\DEBUG{true}

\ifdefined\DEBUG

  \def\rem#1{{\marginpar{\raggedright\scriptsize #1}}}
  \newcommand{\barnr}[1]{\rem{\textcolor{red}{$\bullet$ #1}}}
  \newcommand{\aryar}[1]{\rem{\textcolor{green}{$\bullet$ #1}}}
\else

  \newcommand{\barnr}[1]{}
  \newcommand{\aryar}[1]{}

\fi

\allowdisplaybreaks

\begin{document}

\maketitle

\begin{abstract}

Suppose, we are given a set of $n$ elements to be clustered into $k$ (unknown) clusters, and an oracle/expert labeler that can interactively answer pair-wise queries of the form, ``do two elements $u$ and $v$ belong to the same cluster?''. The goal is to recover the optimum clustering by asking the minimum number of queries. 
In this paper, we initiate a rigorous theoretical study of this basic problem of query complexity of interactive clustering, and provide strong information theoretic lower bounds, as well as nearly matching upper bounds. Most clustering problems come with a similarity matrix, which is used by an automated process to cluster similar points together. Nevertheless, obtaining an ideal similarity function is extremely challenging due to ambiguity in data representation, poor data quality etc., and this is one of the primary reasons that makes clustering hard. To improve accuracy of clustering, a fruitful approach in recent years has been to ask a domain expert or crowd to obtain labeled data interactively. Many heuristics have been proposed, and all of these use a similarity function to come up with a querying strategy. However, there is no systematic theoretical study.

 Our main contribution in this paper is to show the dramatic power of side information aka similarity matrix on reducing the query complexity of clustering. A similarity matrix represents noisy pair-wise relationships such as one computed by some  function on attributes of the elements. A natural noisy model is where similarity values are drawn independently from some arbitrary probability distribution $f_+$ when the underlying pair of elements belong to the same cluster, and from some $f_-$ otherwise. We show that given such a similarity matrix, the query complexity reduces drastically from $\Theta(nk)$ (no similarity matrix) to $O(\frac{k^2\log{n}}{\cH^2(f_+\|f_-)})$ where $\cH^2$ denotes the squared Hellinger divergence. Moreover, this is also information-theoretic optimal within an $O(\log{n})$ factor. Our algorithms are all efficient, and parameter free, i.e., they work without any knowledge of $k, f_+$ and $f_-$, and only depend logarithmically with $n$. Our lower bounds could be of independent interest, and provide a general framework for proving lower bounds for classification problems in the interactive setting. Along the way, our work also reveals intriguing connection to popular community detection models such as the {\em stochastic block model}, significantly generalizes them, and opens up many venues for interesting future research.
 \end{abstract}

\section{Introduction}

Clustering is one of the most fundamental and popular methods for data classification. %Recently, popular heuristic algorithms have been developed that perform clustering with
%the help of an oracle. 
In this paper we initiate a rigorous theoretical study of {\em clustering with the help of an oracle}, a model that saw a recent surge of popular  heuristic algorithms. \let\thefootnote\relax\footnotetext{A prior version of this work appeared in arxiv previously~\cite{mazumdar2016clustering}, see \url{https://arxiv.org/abs/1604.01839}. This paper contains a new efficient Monte Carlo algorithm that has not appeared before, and a stronger lower bound. Some proofs have been rewritten for clarity.}

Suppose we are given a set of $n$ points,  that need to be clustered into $k$ clusters where $k$ is unknown to us. Suppose there is an oracle that either knows the true underlying clustering or can compute the best clustering under some optimization constraints.
We are allowed  to query the oracle whether any two points belong to the same cluster or not. 
How many  such queries are needed to be asked at minimum to perform the clustering exactly? The motivation to this problem lies at the heart of modern machine learning applications where the goal is to facilitate more accurate learning from less data by interactively asking for labeled data, e.g., active learning and crowdsourcing. Specifically, automated clustering algorithms that rely just on a similarity matrix often return inaccurate results. Whereas, obtaining few labeled data adaptively can help in significantly improving its accuracy. Coupled with this observation, clustering with an oracle has generated tremendous interest in the last few years with increasing number of heuristics developed for this purpose \cite{gokhale2014corleone,vesdapunt2014crowdsourcing,dkmr:14,wang2012crowder,wang2013leveraging,fss:16,DBLP:conf/icde/VerroiosG15, dalvi2013aggregating,ghosh2011moderates,
karger2011iterative}. The number of queries is a natural measure of ``efficiency'' here, as it directly relates to the amount of labeled data or the cost of using crowd workers --however, theoretical guarantees on query complexity is lacking in the literature. 

On the theoretical side, query complexity or the decision tree complexity is a classical model of computation that has been extensively studied for different problems \cite{frpu:94,akss:86, bg:90}. For the clustering problem, one can obtain an upper bound of $O(nk)$ on the query complexity easily and it is achievable even when $k$ is unknown \cite{vesdapunt2014crowdsourcing,dkmr:14}:  to cluster an element at any stage of the algorithm, ask one query per existing cluster with this element (this is sufficient due to transitivity), and
start a new cluster if all queries are negative.
It turns out that $\Omega(nk)$ is also a lower bound, even for randomized algorithms (see, e.g., \cite{dkmr:14}). In contrast, the heuristics developed in practice often ask significantly less queries than $nk$. What could be a possible reason for this deviation between the theory and practice? 

Before delving into this question, let us look at a motivating application that drives this work.

\vspace{0.1in}
\noindent {\bf A Motivating Application: Entity Resolution.} Entity resolution (ER, also known as record linkage) is a fundamental problem in data mining and has been studied since 1969 \cite{fellegi1969theory}. The goal of ER is to identify and link/group different manifestations of the same real world object, e.g., different ways of addressing (names, email address, Facebook accounts) the same person, Web pages with different descriptions of the same business, different photos of the same object etc. (see the excellent survey by Getoor and Machanavajjhala \cite{getoor2012entity}). However, lack of an ideal similarity function to compare objects makes ER an extremely challenging task. For example, DBLP, the popular computer science bibliography dataset is filled with ER errors \cite{kopcke2010evaluation}. It is common for DBLP to merge publication records of different persons if they share similar attributes (e.g. same name), or split the publication record of a single person due to slight difference in representation (e.g. Marcus Weldon vs Marcus K. Weldon). In recent years, a popular trend to improve ER accuracy has been to incorporate human wisdom. The works of \cite{wang2012crowder,wang2013leveraging,vesdapunt2014crowdsourcing} (and many subsequent works) use a computer-generated similarity matrix to come up with a collection of pair-wise questions that are asked interactively to a crowd. The goal is to minimize the number of queries to the crowd while maximizing the accuracy. This is analogous to our interactive clustering framework. But intriguingly, as shown by extensive experiments on various real datasets, these heuristics use far less queries than $nk$ \cite{wang2012crowder,wang2013leveraging,vesdapunt2014crowdsourcing}--barring the $\Omega(nk)$ theoretical lower bound. On a close scrutiny, we find that all of these heuristics use some computer generated similarity matrix to guide in selecting the queries. Could these similarity matrices, aka side information, be the reason behind the deviation and significant reduction in query complexity?

Let us call this {\em clustering using side information}, where the clustering algorithm has access to a similarity matrix. This can be generated directly from the raw data (e.g., by applying Jaccard similarity on the attributes), or using a crude classifier which is trained on a very small set of labelled samples. Let us assume the following generative model of side information: a noisy weighted upper-triangular similarity matrix $W=\{w_{i,j}\}$, $1\le i< j \le n$, where $w_{i,j}$ is drawn from a probability distribution $f_+$ if $i,j, i \ne j,$ belong to the same cluster, and else from $f_{-}$. However, the algorithm designer is given only the similarity matrix without any information on $f_+$ and $f_-$. In this work, one of our major contributions is to show the separation in query complexity of clustering with and without such side information. Indeed the recent works of \cite{fss:16,aaai:17} analyze popular heuristic algorithms of \cite{vesdapunt2014crowdsourcing,wang2013leveraging} where the probability distributions are obtained from real datasets which show that these heuristics are significantly suboptimal even for very simple distributions. To the best of our knowledge, before this work, there existed no algorithm that works for arbitrary unknown distributions $f_+$ and $f_-$ with near-optimal performances. We develop a generic framework for proving information theoretic lower bounds for interactive clustering using side information, and design efficient algorithms for arbitrary $f_+$ and $f_-$ that nearly match the lower bound. Moreover, our algorithms are parameter free, that is they work without any knowledge of $f_+$, $f_-$ or $k$.

\vspace{0.1in}
\noindent{\bf Connection to popular community detection models.} The model of side information considered in this paper is a direct and significant generalization of the {\em planted partition model}, also known as the stochastic block model (SBM) \cite{holland1983stochastic, dyer1989solution, decelle2011asymptotic, DBLP:conf/focs/AbbeS15,abh:16,DBLP:conf/colt/HajekWX15,hajek2016achieving,chin2015stochastic,mossel2015consistency}.
The stochastic block model is an extremely well-studied model of random graphs which is used for modeling communities in real world, and is a special case of a similarity matrix we consider. In SBM, two vertices within the same community share an edge with probability $p$, and two vertices in different communities share an edge with probability $q$, that is $f_+$ is ${\rm Bernoulli}(p)$ and $f_-$ is ${\rm Bernoulli}(q)$. % \cite{holland1983stochastic, dyer1989solution, decelle2011asymptotic, abh:16,hajek2015achieving,chin2015stochastic,mossel2015consistency}.
 It is often assumed that $k$, the number of communities, is a constant (e.g. $k=2$ is known as the {\em planted bisection model} and is studied extensively \cite{abh:16,mossel2015consistency,dyer1989solution} or a slowly growing function of $n$ (e.g. $k=o(\log{n})$). The points are assigned to clusters according to a probability distribution indicating the relative sizes of the clusters.  In contrast, not only in our model $f_+$ and $f_-$ can be arbitrary probability mass functions (pmfs), we do not have to make any assumption on $k$ or the cluster size distribution, and can allow for any partitioning of the set of elements (i.e., adversarial setting).  Moreover, $f_+$ and $f_-$ are unknown. For SBM, parameter free algorithms are known relatively recently for constant number of linear sized clusters \cite{abbe2015recovering,hajek2016achieving}.

 There are extensive literature that characterize the  threshold phenomenon in SBM in terms of  $p$ and $q$  for exact and approximate recovery of clusters when relative cluster sizes are known and nearly balanced (e.g., see \cite{DBLP:conf/focs/AbbeS15} and therein for many references). For $k=2$ and equal sized clusters, sharp thresholds are derived in \cite{abh:16,mossel2015consistency} for a specific sparse region of $p$ and $q$ \footnote{Most recent works consider the region of interest as $p=1-\frac{a\log{n}}{n}$ and $q=\frac{b\log{n}}{n}$ for some $a> b >0$.}.   In a more general setting, the vertices in the $i$th  and the $j$th communities are connected with probability $q_{ij}$ and threshold results for the sparse region has been derived in \cite{DBLP:conf/focs/AbbeS15} - our model can be allowed to have this as a special case when we have  pmfs $f_{i,j}$s denoting the  distributions 
of the corresponding random variables. If an oracle gives us some of the pairwise binary relations between elements (whether they belong to  the same cluster or not), the threshold of SBM must also change. But by what amount? This connection to SBM could be of independent interest to study query complexity of interactive clustering with side information, and our work opens up many possibilities for future direction.

%Only recently parameter-free algorithms have been developed for the stochastic block model that work for constant number of nearly-equal sized communities \cite{abbe2015recovering}.  

 Developing lower bounds in the interactive setting appears to be significantly challenging, as algorithms may choose to get any deterministic information adaptively by querying, and standard lower bounding techniques based on Fano-type inequalities \cite{CGT:12,Chen:14} do not apply. One of our major contributions in this paper is to provide a general framework for proving information-theoretic lower bound for interactive clustering algorithms which holds even for randomized algorithms, and even with the full knowledge of $f_+, f_-$ and $k$. In contrast, our algorithms are computationally efficient and are parameter free (works without knowing $f_+,f_-$ and $k$). The technique that we introduce for our upper bounds could be useful for designing further parameter free algorithms which are extremely important in practice.
 
 \vspace{0.1in}
\noindent{\bf Other Related works.}
The interactive framework of clustering model has been studied before
where the oracle is given the entire clustering and the oracle can answer whether a cluster needs to be split
or two clusters must be merged \cite{balcan2008clustering,ABV:14}. Here we contain our attention to pair-wise queries, as in all practical applications that motivate this work \cite{wang2012crowder,wang2013leveraging,gokhale2014corleone,vesdapunt2014crowdsourcing}.  In most cases,  an expert human or crowd serves as an oracle. Due to the scale of the data, it is often not possible for such an oracle to answer queries on large number of input data. Only recently, some heuristic algorithms with $k$-wise queries for small values of $k$ but $k>2$ have been proposed in \cite{DBLP:conf/icde/VerroiosG15}, and a non-interactive algorithm that selects random triangle queries have been analyzed in \cite{vinayak2016crowdsourced}. Perhaps conceptually closest to us is a recent work by Asthiani et al. \cite{ashtiani2016clustering}, that was done independently of ours and appeared subsequent to a previous version of this work~\cite{mazumdar2016clustering}. In  \cite{ashtiani2016clustering}, pair-wise queries for clustering is considered. However, their setting is very different. They consider the specific NP-hard $k$-means objective with distance matrix which must be a metric and must satisfy a deterministic separation property. Their lower bounds are computational and not information theoretic; moreover their algorithm must know the parameters. There exists a significant gap between their lower and upper bounds:$\sim \log{k}$ vs $k^2$, and it would be interesting if our techniques can be applied to improve this.

Here we have assumed the oracle always returns the correct answer. To deal with the possibility that the crowdsourced oracle may %sometimes 
give wrong answers, there are simple majority voting mechanisms or more
complicated techniques~\cite{DBLP:conf/icde/VerroiosG15, dalvi2013aggregating,ghosh2011moderates,
karger2011iterative,chen2016community,vinayak2016crowdsourced} to handle such errors. If we assume the errors are independent-since answers are collected from independent crowdworkers, then we can simply ask each query $O(\log{n})$ times and take the majority vote as the correct answer according to the Chernoff bound. Here our main objective is to study the power of side information, and we do not consider the more complex scenarios of handling erroneous oracle answers. 

\vspace{0.1in}
\noindent{\bf Contributions.} 
%\label{sec:perfect}
%In this section, we consider the clustering problem using  perfect oracle that always returns the correct answers, plus there is side information. Here is the formal description of the problem.
Formally the problem we study in this paper can be described as follows.
\begin{problem}[\cc~with an Oracle]
Consider a set of elements $V\equiv[n]$ with $k$ latent clusters $V_i$, $i =1, \dots, k$, where $k$ is unknown. There is an oracle $\mathcal{O}: V\times V \to \{\pm1\},$ that when queried with a pair of elements $u,v \in V \times V$,  returns $+1$  iff $u$ and $v$ belong to the same cluster, and $-1$ iff $u$ and $v$ belong to different clusters. The queries $Q \subseteq V \times V $ can be done adaptively. 
 Consider the side information $W = \{w_{u,v}: 1 \le u < v \le n\}$,  where the $(u,v)$th entry of  $W$, $w_{u,v}$ is a  random variable drawn from a discrete probability distribution $f_+$ if $u,v$ belong to the same cluster, and is drawn from a discrete\footnote{our lower bound holds  for continuous distributions as well.} probability distribution $f_{-}$\footnote{for simplicity of expression, we treat the sample space to be of constant size. However, all our results extend to any finite sample space  scaling linearly with its size.} if $u,v$ belong to different clusters. The parameters $k, f_{+}$ and $f_{-}$ are unknown.
%\begin{itemize}
%\item {\bf Without Side Information.} Given $V$, find $Q \subseteq V \times V$ such that $|Q|$ is minimum, and from  the oracle answers it is possible to recover $V_i$, $i=1,2,...,k$.
%\item % {\bf Clustering with Side Information.} 
Given $V$ and $W$, find $Q \subseteq V \times V$ such that $|Q|$ is minimum, and from  the oracle answers and $W$  it is possible to recover $V_i$, $i=1,2,...,k$.
%\end{itemize}
\end{problem}

Without side information, as noted earlier, it is easy to see an algorithm with query complexity $O(nk)$ for \cc. When no side information is available, it is also not difficult to have a lower bound of $\Omega(nk)$ on the query complexity. Our main contributions are to develop strong information theoretic lower bounds as well as nearly matching upper bounds  when side information is available, and characterize the effect of side information on query complexity precisely.   

\noindent{\bf Upper Bound (Algorithms).} We show that with side information $W$, a drastic reduction in query complexity of clustering is possible, even  with unknown parameters $f_+$, $f_-$, and $k$. We propose a Monte Carlo randomized algorithm that reduces the number of queries from $O(nk)$ to $O(\frac{k^2 \log n}{\cH^2(f_+\|f_-)})$, where $\cH(f\|g)$ %\equiv D(f_+\|f_-)+D(f_-\|f_+)$ and $D(f\| g)$ is the Kullback-Leibler divergence 
 is the Hellinger divergence
 between the probability distributions $f$, and $g$, and recovers the clusters accurately with high probability (with success probability $1-\frac{1}{n}$) without knowing $f_+$, $f_-$ or $k$ (see, Theorem \ref{thm:div}). Depending on the value of $k$, this could be highly sublinear in $n$.
 Note that, the squared Hellinger divergence between two pmfs $f$ and $g$ is defined to be,
% $$
% \cH^2(f\|g) = \frac{1}{2}\int \Big(\sqrt{f(x)}-\sqrt{g(x)}\Big)^2 dx,
% $$
% when $f$ and $g$ are pdfs, and 
\begin{align*}
 \cH^2(f\|g) = \frac{1}{2}\sum_i \Big(\sqrt{f(i)}-\sqrt{g(i)}\Big)^2.
\end{align*}
 % when $f$ and $g$ represent pmfs.

 We also develop a Las Vegas algorithm, that is one which recovers the clusters with probability $1$ (and not just with high probability), with query complexity $O(n\log{n}+\frac{k^2 \log n}{\cH^2(f_+\|f_-)})$. Since $f_+$ and $f_-$ can be arbitrary, not knowing the distributions provides a major challenge, and we believe, our recipe could be fruitful for designing further parameter-free algorithms. We note that all our algorithms are computationally efficient -  in fact, the time required is bounded by the size of the side information matrix, i.e., $O(n^2)$.  

 \begin{theorem}
 \label{thm:div}
 Let, the number of clusters $k$ be unknown and $f_+$ and $f_-$ be unknown discrete distributions with fixed cardinality of support. There exists an efficient (polynomial-time) Monte Carlo algorithm for \cc~that has query complexity $O(\min{(nk,\frac{k^2\log{n}}{\cH^2(f_+\|f_-)})})$ and recovers all the clusters accurately with probability $1-o(\frac{1}{n})$. Plus there exists an efficient Las Vegas algorithm that with probability $1-o(\frac{1}{n})$ has query complexity $O( n\log n+\min{(nk,\frac{k^2\log{n}}{\cH^2(f_+\| f_-)})})$.
 \end{theorem}
 
\noindent{\bf Lower Bound.} 
Our main lower bound result  is information theoretic, and can be summarized in the following theorem. Note especially that, for lower bound we can assume the knowledge of $k, f_+, f_-$ in contrast to upper bounds, which makes the results stronger. In addition, $f_+$ and $f_-$ can be discrete or continuous distributions. Note that, when $\cH^2(f_+\|f_-)$ is close to $1$, e.g., when the side information is perfect, no queries are required. However, that is not the case in practice, and we are interested in the region where $f_+$ and $f_-$ are ``close'', that is $\cH^2(f_+\|f_-)$ is small.
\begin{theorem}\label{thm:lb-main}  %[\cc~with Perfect Oracle and Side Information]%\label{thm:side}
Assume $\cH^2(f_+\|f_-) \le \frac1{18}$. Any (possibly randomized) algorithm with the knowledge of $f_+, f_-,$ and the number of clusters $k$,  that does not perform $\Omega\Big(\min{\{nk,\frac{k^2}{\cH^2(f_+\|f_-)}\}}\Big)$ expected number of queries, %$\Delta(f_+,f_-)> 0$,
 will be unable to return the correct clustering  with probability at least $\frac1{6}$. And to recover the clusters with probability $1$, the number of queries must be $\Omega\Big(n+\min{\{nk,\frac{k^2}{\cH^2(f_+\|f_-)}\}}\Big)$. %(Proof in Sec.~\ref{sec:perfectlb}).
\end{theorem}

The lower bound therefore matches the query complexity upper bound within a logarithmic factor. 

Note that, when no querying is allowed, this turns out exactly to be the setting of stochastic block model though with much general distributions. We have analyzed this case in Appendix~\ref{sec:zero}. To see how the probability of error must scale, we have used a generalized version of Fano's inequality (e.g.,  \cite{guntuboyina2011lower}). However, when  the number of queries is greater than zero, plus when queries can be adaptive, any such standard technique fails. Hence, significant effort has to be put forth to construct a setting where information theoretic minimax bounds can be applied. This lower bound could be of independent interest, and provides a general framework for deriving lower bounds for fundamental problems of classification, hypothesis testing, distribution testing etc. in the interactive learning setting. 
They may also lead to new lower bound proving techniques in the related multi-round communication complexity model where information again gets revealed adaptively. 

\vspace{0.1in}
\noindent{\bf Organization.} The proof of the lower bound is provided in  Section~\ref{sec:perfectlb}, and the algorithms are presented in Section~\ref{appen:algo}. Section~\ref{sec:perfectub} contains the Monte Carlo algorithm. The Las Vegas algorithm is presented in~\ref{sec:las-vegas}. Generalization of the stochastic block model, as well as exciting future directions are discussed in Appendix~\ref{sec:SBM-general} and \ref{sec:future}.

\vspace{-0.05in}
\section{Lower Bound (Proof of Theorem \ref{thm:lb-main})}
\label{sec:perfectlb}
\vspace{-0.05in}
In this section, we develop our information theoretic lower bounds. We prove a more general result from which Theorem~\ref{thm:lb-main} follows easily.
\begin{lemma}\label{lemma:funda}
Consider the case when we have $k$ equally sized clusters of size $a$ each (that is total number of elements is $n =ka$). Suppose we are allowed to make at most $Q$ adaptive queries
to the oracle. 
The probability of error for any algorithm for \cc~is at least,
$$
1-\frac{2}{ k}\Big(1 +  \sqrt{\frac{4Q}{ak}}\Big)^2-  \frac{4Q}{ak(k-1)}- 2\sqrt{a}\cH(f_+\|f_-) .
$$
\end{lemma}
%If we set $a = \Big\lfloor \frac{1}{9\cH^2(f_+\|f_-)}\Big\rfloor$ in the above lemma, we get if the number of queries $Q < \frac{k^2}{216\cH^2(f_+\|f_-)}$, the probability of error is at least $\frac{1}{6}-O(\frac{1}{\sqrt{k}})$ proving Theorem~\ref{thm:lb-main} (details in Appendix~\ref{appen:lb}). Note that we have not tried to optimize the constants to keep the proof more accessible to the readers.
\vspace{-0.08in}
The main high-level technique to prove Lemma~\ref{lemma:funda} is the following. Suppose, a node is to be assigned to a cluster.
This situation is obviously akin to a  $k$-hypothesis testing problem, and we want to use a lower bound on the probability of error. The side information and the query answers constitute a random vector whose distributions (among the $k$ possible) must be far apart for us to successfully identify the clustering. But the main challenge comes from the interactive nature of the algorithm since it reveals deterministic information and into characterizing the set of elements that are not queried much by the algorithm. 
%This type of idea has found application in a very different context, to design adversarial strategies that lead to lower bounds on average regret for the multi-arm bandit problems \cite{auer2002nonstochastic, cesa2006prediction}), but the problem that we have in hand for lower bound on query-complexity is substantially different than anything considered before. The liberty of an algorithm designer to query freely reveals much more information than a restricted random experiment, and creates the main  challenge.
%However, the problem that we have in hand for lower bound on query-complexity is substantially different than anything considered before. The liberty of an algorithm designer to query freely reveals much more information than a restricted random experiment, and creates the main  challenge. We need to carefully create a subset of clusters, such that while assigning clusters to any element residing in these via a randomized algorithm, we do not make a query with the correct cluster with high probability. We show that $\Omega(k)$ such clusters exist, and each of them can have size about $\frac{1}{\cH^2(f_+\|f_-)}.$ 
\vspace{-0.1in}
\begin{proof}[Proof of Lemma \ref{lemma:funda}]
%The proof basically follows the same argument as in the proof of Theorem~\ref{thm:lb-main}. 
%Since the total number of queries is $Q$ and there are $ka$ elements, there must be an element that is involved in $T \le \frac{2Q}{ka}$ queries. Let that element be $x$.
Since the total number of queries is $Q$, the average number of queries per element is at most $\frac{2Q}{ak}$. Therefore there exist
at least $\frac{ak}{2}$ elements that are queried at most $T < \frac{4Q}{ak}$ times. Let $x$ be one  such element.
We just consider the problem of assignment of $x$ to a cluster (assume,  otherwise the clustering is done), and show that any algorithm will make wrong assignment with positive probability. 
%Let us prove this claim by the following steps.

\noindent{\bf Step 1: Setting up the hypotheses.}
Note that, the side information matrix $W= (w_{i,j})$ is provided where the $w_{i,j}$s are independent random variables.
Now assume the scenario
when we use an algorithm {\rm ALG} to assign $x$ to  one of the $k$ clusters, $V_u, u =1, \dots, k$. %, and all the clusters are fully formed except 
%Note that, for an element $x$, the side informations $w_{i,j}$ where $i \ne x$ and $j\ne x$, do not help in assigning $x$ to a cluster (since in that case $w_{i,j}$ is independent of $x$).
Therefore, given $x$,  {\rm ALG} takes as input the random variables  $w_{i,x}$s where $i \in \sqcup_t V_t$, 
  makes some queries involving $x$ and  outputs a cluster index, which is an assignment for $x$. 
Based on the observations $w_{i,x}$s, the task of  {\rm ALG} is thus a multi-hypothesis testing among $k$ hypotheses.
Let $H_u,u =1, \dots k$ denote the $k$ different hypotheses  $H_u : x \in V_u$. And let  $P_u, u =1, \dots k$ denote the joint probability distributions  of the random matrix $W$ when $x \in V_u$. In short, for any event $\cA$,
$
P_u(\cA) = \Pr(\cA | H_u) 
$.
Going forward, the subscript of probabilities or expectations will denote the appropriate conditional distribution.

\noindent{\bf Step 2: Finding ``weak'' clusters.} 
%Since $x$ is involved in $T$ queries, 
%For an element, the average number of queries made for  \underline{this hypothesis} testing is $T = \frac{Q}{ka}$.
%Since, 
%$
%\avg\{\text{Number of queries made by {\rm ALG}}\} =T,
%$ 
There must exist $t \in \{1,\dots, k\}$ such that,
%$
%\avg_t\{\text{Number of queries made by {\rm ALG}}\} \le T.
%$
$$
\sum_{v=1}^{k} P_t\{\text{ a query made by {\rm ALG} involving cluster } V_v\} \le \avg_t\{\text{Number of queries made by {\rm ALG}}\} \le T.
$$
We now find a subset of clusters, that are ``weak,'' i.e., not queried enough if $H_t$ were true.
Consider the set
$
J' \equiv \{v\in \{1,\dots,k\}:  P_t\{\text{ a query made by {\rm ALG} involving cluster } V_v\} <\frac{2T}{k(1-\beta)}\},
$
where $\beta \equiv \frac{1}{1+\sqrt{\frac{4Q}{ak}}}$. 
We must have, 
$
(k-|J'|)\cdot \frac{2T}{k(1-\beta)} \le T,
$
which implies,
$
|J'| \ge %k - \frac{k(1-\beta)}{2} =
\frac{(1+\beta) k}{2}.
$

Now,  to output a cluster  without using the side information, {\rm ALG}  has to either make a query to the actual cluster the element is from, or query at least $k-1$ times. In any other case, {\rm ALG} must use the side information (in addition to using queries) to output a cluster.
 Let $\cE^u$ denote the event that {\rm ALG} outputs cluster  $V_u$ by using the side information. 
Let $J'' \equiv \{u\in \{1,\dots,k\}: P_t(\cE^u) \le \frac{2}{\beta k} \}.$ 
Since, $\sum_{u=1}^{k} P_t(\cE^u) \le 1,$ we must have, 
$(k- |J''|) \cdot \frac{2}{\beta k} < 1,
\text{ or } |J''| > k - \frac{\beta k}{2}  = \frac{(2-\beta)k}{2}.$
$
\text{We have, }~~~ |J' \cap J''| >  \frac{(1+\beta) k}{2}+\frac{(2-\beta)k}{2}  - k = \frac{k}{2}.
$
This means, $\{V_u: u \in J' \cap J''\}$ contains more than $\frac{ak}{2}$ elements. Since there are $\frac{ak}{2}$ elements that are queried
at most $T$ times, these two sets must have nonzero intersection. Hence, we can assume that, 
 $x \in V_\ell$ for some $\ell \in J' \cap J''$, i.e., let $H_\ell$ be the true hypothesis.
%Assume that we  need to assign  the vertex $j\in V_\ell$ for some $\ell \in J' \cap J''$ to a cluster ($H_\ell$ is the true hypothesis).
Now we characterize the error events of the algorithm ALG in assignment of $x$.

\noindent{\bf Step 3: Characterizing error events for ``$x$''.}
We now consider the following two events.
$$
\cE_1 =  \{\text{a query made by {\rm ALG} involving cluster } V_\ell\};
\cE_2 = \{ k-1 \text{ or more queries were made by {\rm ALG}} \}.
$$

Note that, if the algorithm {\rm ALG} can correctly assign $x$ to a cluster without using the side information then either of $\cE_1$
or $\cE_2$ must have to happen. 
Recall, $\cE^\ell$ denotes the event that {\rm ALG} outputs cluster  $V_\ell$  using the side information.
Now consider the event
$
\cE \equiv \cE^\ell \bigcup \cE_1 \bigcup \cE_2.
$
The probability of correct assignment is at most $P_\ell(\cE).$
We now bound this probability of correct recovery  from above.

\noindent{\bf Step 4: Bounding probability of correct recovery via Hellinger distance.}
We  have,
\begin{align*}
P_\ell(\cE) &\le P_t(\cE) + |P_\ell(\cE) - P_t(\cE)|
 \le P_t(\cE) + \|P_\ell - P_t\|_{TV}
 \le  P_t(\cE) + \sqrt{2} \cH(P_\ell \|P_t),
 \end{align*}
 where, 
 $\|P-Q\|_{TV} \equiv \sup_A |P(A) -Q(A)|$ denotes the total variation distance between two probability distributions $P$ and $Q$ and  
 % we first used the definition of the total variation distance and 
  in the last step we have used
 the relationship between total variation distance and the Hellinger divergence (see, for example, \cite[Eq.~(3)]{sason2016f}).
Now, recall that $P_\ell$ and $P_t$ are the joint distributions of the 
independent random variables $w_{i,x}, i\in \cup_u V_u$. Now, we use the fact  that
squared Hellinger divergence between product distribution of independent random variables are less than 
the sum of the squared Hellinger divergence between the individual distribution. We also note that
 %\ref{lem:chain}, 
 the 
divergence between identical random variables are $0$. We obtain
$$
\sqrt{2\cH^2(P_\ell\| P_t)} \le \sqrt{2\cdot 2a \cH^2(f_+\|f_-)} = 2\sqrt{a}\cH(f_+\|f_-).
$$
 This is true because the only times when $w_{i,x}$ differs under $P_t$ and under $P_\ell$ is when $x \in V_t$ or $x \in V_\ell.$
As a result we have,
$
P_\ell(\cE)  \le P_t(\cE) + 2\sqrt{a}\cH(f_+\|f_-).
$
Now, using Markov inequality $P_t(\cE_2) \le \frac{T}{k-1} \le \frac{4Q}{ak(k-1)}.$ Therefore, 
\begin{align*}
P_t(\cE)& \le P_t(\cE^\ell) + P_t(\cE_1) +P_t(\cE_2) \le \frac{2}{\beta k} +  \frac{8Q}{ak^2(1-\beta)}+  \frac{4Q}{ak(k-1)}.
\end{align*}
Therefore, putting the value of $\beta$ we get, $
P_\ell(\cE)  \le \frac{2}{ k}\Big(1 +  \sqrt{\frac{4Q}{ak}}\Big)^2+  \frac{4Q}{ak(k-1)}+ 2\sqrt{a}\cH(f_+\|f_-),
$
%where we have used the  $\alpha= \frac{1}{1+\sqrt{\frac{8Q}{ak}}}$.
which proves the lemma.
\end{proof}

\vspace{-0.2in}
\begin{proof}[Proof of Theorem \ref{thm:lb-main}] Suppose, 
$a = \lfloor\frac{1}{9\cH^2(f_+\|f_-)}\rfloor$. Then $a \geq 2$, since $\cH^2(f_+\|f_-) \leq \frac{1}{18}$. Also, we can take $nk \ge k^2 a$, since otherwise the theorem is already proved from the $nk$ lower bound.
Consider the situation when we are already given a complete cluster $V_k$ with $n - (k-1)a$ elements,  remaining $(k-1)$
clusters each has 1 element, and the rest $(a-1)(k-1)$
 elements are evenly distributed 
(but yet to be assigned) to the $k-1$ clusters. 
%This means each of the smaller clusters has size $a$ each, and $a \ge 2$. This implies for the following argument to be valid $\cH^2(f_+\|f_-)$ must be at most $\frac{1}{18}$.
%Note that, we assumed the knowledge of  the number of clusters $k$. 
%Also, we can take $nk \ge k^2 a$, since otherwise the
%theorem is already proved from the $nk$ lower bound.
Now we are exactly in the situation of Lemma \ref{lemma:funda} with $k-1$ playing the role of $k$. If we have $Q < \frac{ak^2}{72}$,
The probability of error is at least $1 - o_k(1) - \frac{1}{6} - \frac{2}{3} = \frac{1}{6} - o_k(1)$, where $o_k(1)$ is a term that goes to $0$ with $k$.
% it suffices to have
%a probability of error at least $\frac1{10}$.
 Therefore $Q$ must be  $\Omega(\frac{k^2}{\cH^2(f_+\|f_-)})$. Note that, in this proof we have not in particular tried to optimize the constants.
 
 If we want to recover the clusters with probability $1$, then $\Omega(n)$ is a trivial lower bound. Hence, coupled with the above we get a lower bound of $\Omega(n+\min{\{nk, \frac{k^2}{\cH^2(f_+\|f_-)}\}})$ in that case.
\end{proof}

\section{Algorithms}
\label{appen:algo}
We propose two algorithms (Monte Carlo and Las Vegas) both of which are completely parameter free that is they work without any knowledge of $k, f_+$ and $f_-$, and meet the respective lower bounds within an $O(\log{n})$ factor. We first present the Monte Carlo algorithm which drastically reduces the number of queries from $O(nk)$ (no side information) to $O(\frac{k^2\log{n}}{ \cH^2(f_+\|f_-)})$ and recovers the clusters exactly with probability at least $1-o_{n}(1)$. Next, we present our Las Vegas algorithm.
%makes at most $O(n\log{n}+\frac{k^2\log{n}}{ \cH^2(f_+\|f_-)})$ queries with probability $1-o_n(1)$ but recovers the clusters exactly with probability 1. Whereas the Monte Carlo algorithm makes only $O(\frac{k^2\log{n}}{ \cH^2(f_+\|f_-)})$ queries and recovers the clusters exactly with probability $1-o_{n}(1)$. They both meet the lower bound within an $O(\log{n})$ factor. The Las Vegas algorithm is given in Appendix~\ref{}. Here we describe the Monte Carlo algorithm--a detailed proof of its correctness is given in Appendix~\ref{sec:perfectub}. 

Our algorithm uses a subroutine called {\sf  Membership} that takes as input an element $v \in V$ and a subset of elements $\cC\subseteq  V\setminus\{v\}.$ Assume that $f_+, f_-$ are discrete distributions over fixed set of $q$ points $a_1, a_2, \dots, a_q$; that is $w_{i,j}$ takes value in the set $\{a_1, a_2, \dots, a_q\}.$
Define the empirical ``inter'' distribution  $p_{v,\cC}$ for $i =1,\dots,  q,$
 $
 p_{v,\cC}(i) = \frac{|\{u \in \cC: w_{u,v} = a_i \}|}{|\cC|}
 $
 Also compute the ``intra'' distribution $p_{\cC}$  for $i =1,\dots,  q,$
 $
 p_{\cC}(i) = \frac{|\{(u,v) \in \cC \times \cC: u \ne v, w_{u,v} = a_i \}|}{|\cC|(|\cC|-1)}.
 $
 Then we use {\sf Membership}($v, \cC$) = $-\cH^2( p_{v,\cC} \|  p_{\cC})$ 
as affinity of vertex $v$ to $\cC$, where $\cH( p_{v,\cC} \|  p_{\cC})$ 
denotes the Hellinger divergence between distributions. 
Note that, since the membership is 
always negative, a higher membership implies that the `inter' and `intra' distributions are closer in terms
of Hellinger distance.

Designing a parameter free Monte Carlo algorithm seems to be highly challenging as here, the number of queries depends only logarithmically with $n$. Intuitively, if an element $v$ has the highest membership in some cluster $\calC$, then $v$ should be queried with $\calC$ first. Also an estimation from side information is reliable when the cluster already has enough members. Unfortunately, we know neither whether the current cluster size is reliable, nor we are allowed to make even one query per element.

To overcome this bottleneck, we propose an iterative-update algorithm which we believe will find more uses in developing parameter free algorithms. We start by querying a few points so that there is at least one cluster with $\Theta(\log{n})$ points. Now based on these queried memberships, we learn two empirical distributions $p_{+}^{1}$ from intra-cluster similarity values, and $p_{-}^{1}$ from inter-cluster similarity values. Given an element $v$ which has not been clustered yet, and a cluster $\cC$ with the highest number of current members, we would like to consider the submatrix of side information pertaining to $v$ and all $u \in \cC$ and determine whether that side information is generated from $f_+$ or $f_-$. We know if the statistical distance between $f_+$ and $f_-$ is small, then we would need more members in $\cC$ to successfully do this test. Since, we do not know $f_+$ and $f_-$, we compute the squared Hellinger divergence between $p_{+}^{1}$ and $p_{-}^{1}$, and use that to compute a threshold $\tau_1$ on the size of $\cC$. If $\cC$ crosses this size threshold, we just use the side information to determine if $v$ should belong to $\cC$. Otherwise, we query further until there is one cluster with size $\tau_1$, and re-estimate the empirical distributions  $p_{+}^{2}$ and $p_{-}^{2}$. Again, we recompute a threshold $\tau_2$, and stop if the cluster under consideration crosses this new threshold. If not we continue. Interestingly, we can show when the process converges, we have a very good estimate of $\cH(f_+\|f_-)$ and, moreover it converges fast. 

%In this section, we first prove the correctness of the Monte Carlo algorithm, and then give the Las Vegas algorithm along with its proof.

%The first step for the algorithm is to compute an approximation of $f_+$ and $f_-$, say $p_+$ and $p_-$ respectively by querying a few items. It is quite possible that $p_+$ and $p_-$ are crude approximations of $f_+$ and $f_-$ and working with them will lead to erroneous clustering. Interestingly, we show that by an iterative update and estimate process, we can obtain a $p_+$ and $p_-$ which are close to $f_+$ and $f_-$. Then we try to assign an element to a cluster by computing the empirical distribution from the side information matrix.  There will be a grey region where we cannot be confident to either include or discard a vertex. We show that since the range of that region is small, the number of elements that will fall in the grey region is negligible, and we can query for them to resolve their true memberships.

\subsection{Monte Carlo Algorithm}
\label{sec:perfectub}
%We now design an algorithm that is completely parameter free, that is it has no knowledge of $k$, $f_+$, or $f_-$ and recovers all the clusters accurately with high probability\footnote{the algorithm works even with different $f_{i,i}$s}. The query complexity of the algorithm matches the worst case bound within an $O(\log{n})$ factor. Note that, as a side information, we are given the noisy similarity matrix $W$.
%
%Assume that $f_+, f_-$ are discrete distributions over $q$ points $a_1, a_2, \dots, a_q$; that is $w_{i,j}$ takes value in the set $\{a_1, a_2, \dots, a_q\}.$ We will treat $q$ as a constant for simplicity of expression, otherwise all query complexity results scale by a factor of $q$.
%The algorithm uses a subroutine called {\sf  Membership} that takes as input an element $v \in V$ and a subset of elements $\cC\subseteq  V\setminus\{v\}.$
%
%Compute the `inter' distribution  $p_{v,\cC}$ for $i =1,\dots,  q,$
% $
% p_{v,\cC}(i) = \frac{1}{|\cC|} \cdot |\{u \in \cC: w_{u,v} = a_i \}|.
% $
% 
% Also compute the `intra' distribution $p_{\cC}$  for $i =1,\dots,  q,$
% $
% p_{\cC}(i) = \frac{1}{|\cC|(|\cC|-1)} \cdot |\{(u,v) \in \cC \times \cC: u \ne v, w_{u,v} = a_i \}|.
% $
%Then define {\sf Membership}($v, \cC$) = $- \cH(p_{v,\cC} \|  p_{\cC}).$ Note that, since the membership is 
%always negative, a higher membership implies that the `inter' and `intra' distributions are closer in terms
%of the Hellinger distance.

The algorithm has several phases. 

\vspace{0.1in}
\noindent{\bf Phase 1. Initialization.}
We initialize the algorithm by selecting any vertex $v$ and creating a singleton cluster $\{v\}$. We then keep selecting new vertices randomly and uniformly that have not yet been clustered, and query the oracle with it by choosing exactly one vertex from each of the clusters formed so far. If the oracle returns $+1$ to any of these queries then we include the vertex in the corresponding cluster, else we create a new singleton cluster with it. We continue this process until at least one cluster has grown to a size of $\lceil C\log{n}\rceil$, where $C$ is an appropriately chosen constant depending on $q$ \footnote{the precise value of $C$ can be deduced from the proof given $q$}.

\begin{observation} 
\label{observe:1}
{\it The number of queries made in Phase 1 is at most $O(k^2\log{n})$.}\end{observation}
\begin{proof} We stop the process as soon as a cluster has grown to size of $\lceil C\log{n}\rceil$. Therefore, we may have clustered at most $k*\lceil C\log{n}\rceil$ vertices at this stage, each of which may have required $k$ queries to the oracle, one for every cluster.
\end{proof}

\noindent{\bf Phase 2. Iterative Update.} 
Let $\cC_1,\cC_2,...\cC_{l_x}$ be the set of clusters formed after the $x$th iteration for some $l_x \leq k$, where we consider Phase $1$ as the $0$-th iteration.  We estimate $$p_{+,x}= \frac{1}{\sum_{i=1}^{l_x}{{|\cC_i|}\choose {2}}} \cdot |\{u, v \in \cC_i: w_{u,v} = a_i \}|, \text{ and }$$ 
$$p_{-,x}=\frac{1}{\sum_{i=1}^{l_x}\sum_{j < i}|\cC_i||\cC_j|}  \cdot |\{u \in \cC_i, v \in \cC_j , i < j, i,j \in [1,l_x]: w_{u,v} = a_i \}|$$

Define $$M^{E}_{x} = \frac{C\log n}{\cH(p_{+,x}\| p_{-,x})^2}.$$

If there is no cluster of size at least $M^{E}_{x}$ formed so far, we select a new vertex yet to be clustered and query it exactly once with the existing clusters (that is by selecting one arbitrary point from every cluster and querying the oracle with the new vertex and the selected one), and include it in an existing cluster or create a new cluster with it based on the query answer. We then set $x=x+1$ and move to the next iteration to get updated estimates of $p_{+,x}, p_{-,x}, M^{E}_{x}$ and $l_x$.

Else if there is a cluster of size at least $M^{E}_{x}$, we stop and move to the next phase.

\noindent{\bf Phase 3. Processing the grown clusters.} Once Phase $2$ has converged, let $p_+, p_-, \cH(p_+\| p_-), M^E$ and $l$ be the final estimates. For every cluster $\cC$ of size $|\cC| \geq M^E$, we call it {\sf grown} and we do the following. 
 
 (3A.) For every unclustered vertex $v$, if ${\sf Membership(v,\cC)} \geq -(\frac{4\cH(p_+\|p_-)}{C}-\frac{2\cH(p_+\|p_-)^2}{C\sqrt{\log{n}}})$, then we include $v$ in $\cC$ {\em without} querying. 
 
 (3B.) We create a new list ${\sf Waiting}(\cC)$, initially empty. If $$-(\frac{4\cH(p_+\|p_-)}{C}-\frac{2\cH(p_+\|p_-)^2}{C\sqrt{\log{n}}}) > {\sf Membership(v,\cC)} \geq -(\frac{4\cH(p_+\|p_-)}{C}+\frac{2\cH(p_+\|p_-)^2}{C\sqrt{\log{n}}}),$$ then we include $v$ in ${\sf Waiting}(\cC)$.
 For every vertex in ${\sf Waiting}(\cC)$, we query the oracle with it by choosing exactly one vertex from each of the clusters formed so far starting with $\cC$. If oracle returns answer ``yes'' to any of these queries then we include the vertex in that cluster, else we create a new singleton cluster with it. We continue this until ${\sf Waiting}(\cC)$ is exhausted.

We then call $\cC$ completely grown, remove it from further consideration, and move to the next grown cluster. if there is no other grown cluster, then we move back to Phase $2$.

% If after this, $|\cC| < 2\log{n}M^E$ then, we move to the next grown cluster. If there is no more grown cluster, we move to Phase 2 and repeat.
% 
% Else, $|\cC| \geq 2\log{n}M^E$. We partition $\cC$ into $\lceil {\log{n}} \rceil$ parts of equal size $D_1,D_2,...,D_{\lceil {\log{n}} \rceil}$. For every vertex $v \in {\sf Waiting}(\cC)$ we estimate ${\sf Membership}(c,D_i)$, $i=1,2,..,{\lceil \log{n} \rceil}$. If for at least one of the estimates, the returned value $\geq -(\frac{4\cH(p_+\|p_-)}{C}-\frac{2\cH(p_+\|p_-)^2}{C\sqrt{\log{n}}})$, then we include $v$ in $\cC$ {\em without} querying. On the other hand, if all the estimates return value $< -(\frac{4\cH(p_+\|p_-)}{C}+\frac{2\cH(p_+\|p_-)^2}{C\sqrt{\log{n}}})$, we discard $v$ from ${\sf Waiting(\cC)}$.
% 
% We then move to the next grown cluster, and if there is none, we move back to Phase $2$ to grow a different cluster.
% 
% \noindent{\bf Phase 4. Querying the Left-over Vertices}
% 
% If there is any vertex which remains to be clustered, we query the oracle with it by choosing exactly one vertex from each of the clusters formed so far. If oracle returns answer ``yes'' to any of these queries then we include the vertex in that cluster, else we create a new singleton cluster with it. We continue this until all the vertices are clustered.
 
 \subsection{Analysis.} 
 
One of the important tools that will be used in this section is Sanov's theorem from the large-deviation theory.
\begin{lemma}[Sanov's theorem] \label{thm:sanov}
Let $X_1, \dots, X_n$ are  iid random variables with a finite sample space $\cX$ and distribution $P$. Let $P^n$ denote their joint distribution.
Let $E$ be a set of probability distributions on $\cX$. The empirical distribution $\tilde{P}_n$ gives probability $\tilde{P}_n(\cA) = \frac{1}{n}\sum_{i=1}^n {\bf 1}_{X_i \in \cA}$ to any event $\cA$. Then,
$$
P^n(\{x_1,\dots, x_n\}: \tilde{P}_n \in E) \le (n+1)^{|\cX|} \exp(-n \min_{P^\ast\in E}D(P^\ast\|P)).
$$
\end{lemma}
A continuous version of Sanov's theorem is also possible, especially when the set $E$ is convex (as a matter of fact the polynomial term in front of the right hand side can be omitted in cerain cases), but we omit here for clarity. The Sanov's theorem states, if we have an empirical distribution $P^n$ and a set of all distributions satisfying certain property $E$, then the probability $P^n \in E$ decreases exponentially with the minimum KL divergence of $P^n$ with any distribution in $E$. Note that, the KL divergence in the exponent of the Sanov's theorem naturally indicates an upper bound in terms of KL divergence. However, a major difficulty in dealing with KL divergence is that it is not a distance and does not satisfy triangle inequality. We overcome that by dealing with Hellinger distance instead. 

 There are two parts to the analysis, showing the clusters are correct with high probability and determining the query complexity.

 \vspace{0.1in}
%\noindent{\bf Correctness of clusters.}
\begin{lemma}
\label{lemma:correct1}
With probability at least $1-\frac{6}{n^3}$ all of the following holds for an appropriately chosen constant $B$

(a) $\cH(p_+\|f_+) \leq \frac{2\cH(p_+\|p_-)^2}{B\sqrt{\log{n}}}$

(b) $\cH(p_-\|f_-) \leq \frac{2\cH(p_+\|p_-)^2}{B\sqrt{\log{n}}}$

(c) $\cH(p_+\|p_-)\left(1+\frac{4\cH(p_+\|p_-)}{B\sqrt{\log{n}}} \right)\geq \cH(f_+\|f_-) \geq \cH(p_+\|p_-)\left(1-\frac{4\cH(p_+\|p_-)}{B\sqrt{\log{n}}} \right)$
\end{lemma}
\begin{proof}
 Let $\cC$ be a cluster which according to the updated estimates of $p_+$ and $p_-$ has crossed the updated $M^E$ threshold. Since $|\cC| \geq M^E$, $p_+$ is estimated based on at least ${{M^E}\choose{2}}$ edges. We assume the largest cluster size in the input instance is at most $\frac{n}{2}$\footnote{We could have also assumed the largest cluster size is at most $n(1-\epsilon)$ for some constant $\epsilon >0$ and adjust the constants appropriately.}. Suppose the total number of vertices selected in Phase $1$ and Phase $2$ before $\cC$ grew to $M^E$ is strictly less than $\frac{3M^E}{2}$. Then the expected number of vertices selected from $\cC$ is at most $\frac{3M^E}{4}$. Then, by the Chernoff bound, the probability that the number of vertices selected from $\cC$ is $M^E$ is at most $e^{-\frac{M^E}{36}}$. Taking $C \geq 118$, we get with probability at least $1-\frac{1}{n^3}$, the number of vertices chosen from outside $\cC$ is at least $\frac{M^E}{2}$.  Thus, $p_-$ is estimated based on at least $\frac{(M^E)^2}{2}$ edges. 
  
 Here, we use the following version of the Chrenoff bound\footnote{note that the version of the Chernoff bound also holds for sampling without replacement, which is the case here \cite{hoeffding1963probability}.}.
 
 \begin{lemma}[The Chernoff Bound] Let $X_1, X_2,...., X_n$ be independent random variable taking values in $\{0,1\}$ with $E[X_i]=p_i$. Let $X=\sum_{i=1}^{n} X_i$, and $\mu =E[X]$. Then the following holds
 \begin{enumerate}
 \item For $0 < \delta \leq 1$, $Pr[X \leq (1-\delta)\mu] \leq e^{-\mu\delta^2/2}$
 \item For $0 < \delta \leq 1$, $Pr[X \geq (1+\delta)\mu] \leq e^{-\mu\delta^2/3}$
 \end{enumerate}
 \end{lemma}

(a)  Let $M={{M^E}\choose{2}}\geq \frac{(M^E)^2}{3}$. Now, select $\delta=\sqrt{\frac{C'\log{n}}{M}}$, where $C'$ is a constant that ensures $n^{2C'} \geq n^{\frac{8\sqrt{C'}}{27\sqrt{3}}-6} \geq (M+1)^{q}\approx (M^E+1)^{2q}$, also $C' \geq 3$.
 \begin{align*}
 \Pr\Big(\cH(p_+ \| f_+) \geq \delta\Big) &=f_+\Big(\{ p_+: \cH(p_+\|f_+) \geq \delta\Big)\\
 & \hspace{-2in} = (M+1)^{q} \exp(-M \underset{p: \cH(p\| f_+) \ge \delta}\min D(p\| f_+)),
\end{align*}

Here in the last step we have used Sanov's theorem (see, Lemma \ref{thm:sanov}). Using the relationship between KL-divergence and Hellinger distance, we get
\begin{align*}
D(p\| f_+) \geq 2\cH^2(p\|f_+) \geq 2\delta^2
\end{align*}
where in the last step we used the optimization condition under the Sanov's theorem. Setting $\delta=\sqrt{\frac{C'\log{n}}{M}}$, $M \geq \frac{(M^E )^2}{3}=\frac{C^2\log^2{n}}{3\cH(p_+\|p_-)^4}$, we get $\delta=\frac{\sqrt{3C'}\cH(p_+\|p_-)^2}{C\sqrt{\log{n}}}$.
Let us take $B'=\frac{C}{\sqrt{3C'}}$, and $B=\sqrt{\frac{C}{C'}}$, we have $B \leq B'$ and we get
\begin{align*}
 \Pr\Big(\cH(p_+ \| f_+) \geq  \frac{2\cH(p_+\|p_-)^2}{B'\sqrt{\log{n}}}\Big) \leq \frac{1}{n^3}
  \end{align*}
  
  Hence,
  \begin{align*}
 \Pr\Big(\cH(p_+ \| f_+) \geq  \frac{2\cH(p_+\|p_-)^2}{B\sqrt{\log{n}}}\Big) \leq \frac{1}{n^3}
  \end{align*}
 
 (b) Following a similar argument as above, we get
 \begin{align*}
 \Pr\Big(\cH(p_- \| f_-) \geq  \frac{2\cH(p_+\|p_-)^2}{B\sqrt{\log{n}}}\Big) \leq \frac{1}{n^3}
 \end{align*}
 
 (c) Now 
 \begin{align*}
 \cH(f_+ \|f_-) & \geq \cH(p_+\|p_-)- \cH(p_+\|f_+)-\cH(p_-\|f_-) ~~\text{ by applying triangle inequality }\\
 & \geq \cH(p_+\|p_-)-\frac{4\cH(p_+\|p_-)^2}{B\sqrt{\log{n}}} ~~\text{ from (a) and (b) with probability at least $1-\frac{2}{n^3}$}\\
 &=\cH(p_+\|p_-)\Big(1-\frac{4\cH(p_+\|p_-)}{B\sqrt{\log{n}}}\Big)
 \end{align*}
 Similarly, 
  \begin{align*}
  \cH(p_+\|p_-) &\geq \cH(f_+ \| f_-)-\cH(p_+\|f_+)-\cH(p_-\|f_-)~~\text{by triangle inequality}\\
  &\geq  \cH(f_+ \| f_-)-\frac{4\cH(p_+\|p_-)^2}{B\sqrt{\log{n}}}~~\text{from (a) and (b) with probability at least $1-\frac{2}{n^3}$}
  \end{align*}
  
Hence, by union bound all of (a), (b) and (c) hold with probability at least $1-\frac{6}{n^3}$.  
 \end{proof}
   
 \begin{lemma}
 \label{lemma:correct2}
 Let $\cC$ be a cluster considered in Phase $3$ of size at least $M^E$ then the following holds with probability at least $1-o_{n}(1)$.
 
 (a) If ${\sf Membership}(v,\cC) > -(\frac{\cH(p_+\|p_-)}{B}-\frac{2\cH(p_+\|p_-)^2}{B\sqrt{\log{n}}})$ then $v$ is in $\cC$ 
 
 (b) If $v \in \cC$ then ${\sf Membership}(v,\cC) \geq -(\frac{\cH(p_+\|p_-)}{B}+\frac{2\cH(p_+\|p_-)^2}{B\sqrt{\log{n}}})$
 
 \end{lemma}
 \begin{proof}
 Suppose $v \in \cC$. Then for any $\delta >0$, we have
 \begin{align*}
 \Pr\Big(\cH(p_{v,\cC} \| f_+) > \delta \mid v \in \cC\Big) &=f_+\Big(\cH(p_{v,\cC} \| f_+) > \delta \Big)\\
 &\leq (M^E+1)^q \exp(-M^E \underset{p: \cH(p\| f_+) \ge \delta}\min D(p\| f_+)) ~~\text{(by Sanov's theorem)}\\
 &\leq (M^E+1)^q \exp(-M^E \underset{p: \cH(p\| f_+) \ge \delta}\min 2\cH^2(p\|f_+)) \\
 &~~\text{(noting the relationship between KL-divergence and Hellinger distance)}\\
 &\leq (M^E+1)^q \exp(-2 M^E \delta^2)
 \end{align*}
 
 Setting $M^E\delta^2=C'\log{n}$, we get $\delta=\sqrt{\frac{C'\log{n}}{M^E}}=\sqrt{\frac{C'}{C}}\cH(p_+\|p_-) = \frac{\cH(p_+\|p_-)}{B}$ (by noting the value of $B$), we get
 
 \begin{align*}
 \Pr\Big(\cH(p_{v,\cC} \| f_+) > \frac{\cH(p_+\|p_-)}{B}  \mid v \in \cC\Big) \leq \frac{1}{n^3}~~~ \text{(by noting the value of $C'$)}
 \end{align*}
 
 Similarly, 
 \begin{align*}
 \Pr\Big(\cH(p_{v,\cC} \| f_-) > \frac{\cH(p_+\|p_-)}{B}  \mid v \not \in \cC\Big) \leq \frac{1}{n^3}
 \end{align*}
 
 Therefore, with at least $1-\frac{2}{n^2}$ probability (by applying union bound over all $v$ the following hold.
 (i) If $v \in \cC$ then $\cH(p_{v,\cC} \| f_+) < \frac{\cH(p_+\|p_-)}{B}$ and
 (ii) If $v \not \in \cC$ then $\cH(p_{v,\cC} \| f_-) < \frac{\cH(p_+\|p_-)}{B}$.
 
 (a)~We have ${\sf Membership}(v,\cC) > -(\frac{\cH(p_+\|p_-)}{B}-\frac{2\cH(p_+\|p_-)^2}{B\sqrt{\log{n}}})$, that is $\cH(p_{v,\cC} \|p_+) < \frac{\cH(p_+\|p_-)}{B}-\frac{2\cH(p_+\|p_-)^2}{B\sqrt{\log{n}}}$.  Suppose if possible $v \not \in \cC$. Then, we have
 \begin{align*}
 \cH(p_{v,\cC} \| f_+)& \leq \cH(p_{v,\cC}\|p_+)+\cH(p_+\|f_+) ~~\text{by triangle inequality}\\
& < \frac{\cH(p_+\|p_-)}{B}-\frac{2\cH(p_+\|p_-)^2}{B\sqrt{\log{n}}}+\cH(p_+\|f_+) ~~\text{applying condition on ${\sf Membership}(v,\cC)$}\\
& \leq \frac{\cH(p_+\|p_-)}{B} ~~\text{from Lemma \ref{lemma:correct1} (a) with probability at least $1-\frac{1}{n^3}$}
 \end{align*}
 
 Then we have,
 \begin{align*}
 \cH(p_{v,\cC} \| f_-)&\geq \cH(f_+\|f_-)-\cH(p_{v,\cC} \| f_+) ~~\text{by triangle inequality}\\
 & \geq \cH(p_+\|p_-)-\frac{4\cH(p_+\|p_-)^2}{B\sqrt{\log{n}}}-\cH(p_{v,\cC} \| f_+) ~~\text{from Lemma \ref{lemma:correct1} (c) with probability at least $1-\frac{2}{n^3}$}\\
 & \geq \Big(1-\frac{1}{B}\Big)\cH(p_+\|p_-)-\frac{4\cH(p_+\|p_-)^2}{B\sqrt{\log{n}}} ~~\text{with probability at least $1-\frac{3}{n^3}$}\\
 & \geq \Big(1-\frac{1}{B}-\frac{4}{B\sqrt{\log{n}}}\Big)\cH(p_+\|p_-)~~\text{since $\cH(p_+\|p_-) \leq 1$}\\
 & > \frac{\cH(p_+\|p_-)}{B} ~~\text{by taking $B > 6$, or $C \geq 36C'$}
 \end{align*}
 
 This contradicts that $v \not \in \cC$.

 (b)~Now assume $v \in \cC$ but ${\sf Membership}(v,\cC) \geq -(\frac{\cH(p_+\|p_-)}{B}+\frac{2\cH(p_+\|p_-)^2}{B\sqrt{\log{n}}})$, that is $\cH(p_{v,\cC} \|p_+)\geq \frac{\cH(p_+\|p_-)}{B}+\frac{2\cH(p_+\|p_-)^2}{B\sqrt{\log{n}}}$. We have
 \begin{align*}
 \cH(p_{v\cC}\|f_+) & \geq \cH(p_{v\cC}\|p_+) -\cH(f_+\|p_+) \\
 & \geq \frac{\cH(p_+\|p_-)}{B}+\frac{2\cH(p_+\|p_-)^2}{B\sqrt{\log{n}}}-\cH(f_+\|p_+)~~\text{ applying condition on ${\sf Membership}(v,\cC)$}\\
 & \geq \frac{\cH(p_+\|p_-)}{B} ~~\text{from Lemma \ref{lemma:correct1} (a) with probability at least $1-\frac{1}{n^3}$}
 \end{align*}
 This contradicts the fact that $v \in \cC$.
 \end{proof}
  \begin{corollary}
  \label{corollary:correct}
  Let $\cC$ be a cluster considered in Phase $3$ of size at least $M^E$ then the following hold with probability at least $1-\frac{2}{n^2}$.
  
 (a) Vertices that are included in $\cC$ in Phase $(3A)$ truly belong to $\cC$.
 
 (b) Vertices that are not in ${\sf Waiting}(\cC)$ can not be in $\cC$.
 \end{corollary}
 \begin{proof}
 Follows from Lemma \ref{lemma:correct2} (a) and (b) respectively.
 \end{proof}
 
 \begin{lemma}
 \label{lemma:anti1}
  Let $\cC$ be a cluster considered in Phase $3$ of size at least $M^E$ and $\hat{\cC}$ denotes the true cluster with $\cC \subseteq \hat{\cC}$. Then after Phase $(3A)$, $|\hat{\cC} \setminus \cC|=o(1)$ with probability at least $1-\frac{1}{n^2}$.
  \end{lemma}
  \begin{proof}
  We have from Lemma \ref{lemma:correct2} that for $v$ to belong to $\hat{\cC}$, it must satisfy  ${\sf Membership}(v,\cC) \geq -(\frac{\cH(p_+\|p_-)}{B}+\frac{2\cH(p_+\|p_-)^2}{B\sqrt{\log{n}}})$. On the otherhand, if $v$ has ${\sf Membership}(v,\cC) > -(\frac{\cH(p_+\|p_-)}{B}-\frac{2\cH(p_+\|p_-)^2}{B\sqrt{\log{n}}})$ then $v$ has already been included in $\cC$. Therefore, the grey region of ${\sf Membership}(v,\cC)$ values for which we cannot decide on whether or not to include $v$ to $\cC$ is when ${\sf Membership}(v,\cC) \in -\frac{\cH(p_+\|p_-)}{B} \pm \frac{2\cH(p_+\|p_-)^2}{B\sqrt{\log{n}}}$, that is $\cH(p_{v,\cC} \| p_+) \in \frac{\cH(p_+\|p_-)}{B} \pm \frac{2\cH(p_+\|p_-)^2}{B\sqrt{\log{n}}}$.
  
  Now,
\begin{align*}
& \Pr\Big(\cH(p_{v,\cC} \| p_+) \in \frac{\cH(p_+\|p_-)}{B} \pm \frac{2\cH(p_+\|p_-)^2}{B\sqrt{\log{n}}}\Big) \leq \Pr\Big(\cH(p_{v,\cC} \| p_+) \geq \frac{\cH(p_+\|p_-)}{B} - \frac{2\cH(p_+\|p_-)^2}{B\sqrt{\log{n}}}\Big)\\
& \leq (M^E+1)^q \exp\Big(-M^E \underset{p: \cH(p\| p_+) \geq \frac{\cH(p_+\|p_-)}{B} - \frac{2\cH(p_+\|p_-)^2}{B\sqrt{\log{n}}} }\min D(p\| f_+)\Big)~~\text{by Sanov's theorem}
\end{align*}
  
  Now, 
  \begin{align*}
  & D(p\| f_+) \geq 2\cH(p \| f_+)^2  \geq 2\Big( \cH(p\|p_+) - \cH(p_+\|f_+)\Big)^2 ~~\text{by triangle inequality}\\
  & \geq 2\Big( \frac{\cH(p_+\|p_-)}{B} - \frac{2\cH(p_+\|p_-)^2}{B\sqrt{\log{n}}} - \cH(p_+\|f_+)\Big)^2~~\text{from the optimization condition}\\
  & \geq 2\Big( \frac{\cH(p_+\|p_-)}{B} - \frac{4\cH(p_+\|p_-)^2}{B\sqrt{\log{n}}}\Big)^2~~\text{from Lemma~\ref{lemma:correct1} (a) with probability at least $1-\frac{1}{n^3}$}\\
  &=\frac{2\cH^2(p_+\|p_-)}{B^2}\Big(1-\frac{4\cH(p_+\|p_-)}{B}\Big)^2\geq \frac{2\cH^2(p_+\|p_-)}{B^2}\Big(1-\frac{4}{B}\Big)^2\\
  &\geq \frac{2\cH^2(p_+\|p_-)}{27 B}~~\text{by inserting the minimum value for $\frac{1}{B}\Big(1-\frac{4}{B}\Big)^2$}
  \end{align*}
  
  Now $M^E \geq \frac{C\log{n}}{\cH(p_+\|p_-)^2}$. Hence,
  \begin{align*}
& \Pr\Big(\cH(p_{v,\cC} \| p_+) \in \frac{\cH(p_+\|p_-)}{B} \pm \frac{2\cH(p_+\|p_-)^2}{B\sqrt{\log{n}}}\Big) \\
& \leq (M^E+1)^q \exp(-\frac{2C}{27B}\log{n})+\frac{1}{n^3}=(M^E+1)^q \exp(-\frac{4\sqrt{C'}}{27\sqrt{3}}\log{n})+\frac{1}{n^3} \leq \frac{2}{n^3}\end{align*}

Hence the expected number of vertices $v \in \cC$ in the grey region is $\leq \frac{2}{n^2}$. Thus by simple Markov inequality, after Phase $(3A)$, the probability that $|\hat{\cC} \setminus \cC| \geq 4$ is at most $\frac{1}{2n^2}$. Hence, with probability at least $1-\frac{1}{2n^2}$, the size is bounded by $4$.

  \end{proof}
  
  \begin{lemma}
  \label{lemma:query-complexity}
  The algorithm asks at most $O(\frac{k^2\log{n}}{\cH(f_+\|f_-)^2})$ queries over the three phases with probability $1-o_{n}(1)$.
  \end{lemma}
  \begin{proof}
  In Phase $1$, as seen from Observation \ref{observe:1}, the number of queries is $O(k^2\log{n}) \leq O(\frac{k^2\log{n}}{\cH(f_+\|f_-)^2})$, as $0\leq \cH(f_+\|f_-)^2 \leq 1$.
  
  In Phase $2$, from Lemma \ref{lemma:correct1}, at any time when we have a grown cluster
  \begin{align*}
  \cH(p_+\|p_-) &\geq \cH(f_+ \| f_-)-\cH(p_+\|f_+)-\cH(p_-\|f_-)~~\text{by triangle inequality}\\
  &\geq  \cH(f_+ \| f_-)-\frac{4\cH(p_+\|p_-)^2}{B\sqrt{\log{n}}}~~\text{from Lemma~\ref{lemma:correct1}}
  \end{align*}
  Therefore, 
  \begin{align*}
  \cH(p_+\|p_-) \geq \frac{\cH(f_+ \| f_-)}{1+\frac{4\cH(p_+\|p_-)}{B\sqrt{\log{n}}}}\geq \frac{\cH(f_+ \| f_-)}{2}
    \end{align*}
    
    This also shows whenever one cluster has grown to a size of $\frac{4C\log{n}}{\cH^2(f_+\|f_-)}$, then $M^E$ must cross the threshold based on the newest estimate of $p_+$ and $p_-$. Hence, Phase $2$ never grows a cluster beyond a size of $O(\frac{\log{n}}{\cH^2(f_+\|f_-)})$ with probability $1-\frac{1}{n^3}$. Hence, in Phase $2$, the total number of queries can be at most $O\Big(\frac{k^2\log{n}}{\cH^2(f_+ \| f_-)}\Big)$.
    
    In Phase $3$, the total number of queries made is at most $O(k^2)$ with probability at least $1-\frac{1}{2n}$ due to Lemma \ref{lemma:anti1}, and applying union bound over all the clusters.
    
    Thus, we get the overall query complexity is $O(\frac{k^2\log{n}}{\cH(f_+\|f_-)^2})$ with probability $1-o_{n}(1)$, where $o_{n}(1)$ denotes a function of $n$ that goes to $0$ with $n$.
  \end{proof}
  
Putting together all the lemmas, we arrive at the statement of Theorem  \ref{thm:div}.

\subsection{A Las Vegas Algorithm for  \cc~with an Oracle}
\label{sec:las-vegas}

 In this section, we design a Las Vegas algorithm for clustering with oracle. 

Recall that, our algorithm uses a subroutine called {\sf  Membership} that takes as input an element $v \in V$ and a subset of elements (cluster) $\cC\subseteq  V\setminus\{v\}.$
Assume that $f_+, f_-$ are discrete distributions over $q$ points $a_1, a_2, \dots, a_q$; that is $w_{i,j}$ takes value in the set $\{a_1, a_2, \dots, a_q\}.$
We defined the empirical ``inter'' distribution  $p_{v,\cC}$ for $i =1,\dots,  q,$
 $
 p_{v,\cC}(i) = \frac{1}{|\cC|} \cdot |\{u \in \cC: w_{u,v} = a_i \}|.
 $
 Also compute the ``intra'' distribution $p_{\cC}$  for $i =1,\dots,  q,$
 $
 p_{\cC}(i) = \frac{1}{|\cC|(|\cC|-1)} \cdot |\{(u,v) \in \cC \times \cC: u \ne v, w_{u,v} = a_i \}|.
 $
 Then we use {\sf Membership}($v, \cC$) = $-\cH^2( p_{v,\cC} \|  p_{\cC})$
as affinity of vertex $v$ to cluster $\cC$, where $\cH( p_{v,\cC} \|  p_{\cC})$ 
denotes the Hellinger divergence between distributions. 
Note that, since the membership is 
always negative, a higher membership implies that the `inter' and `intra' distributions are closer in terms
of Hellinger distance.

The algorithm works as follows. Let $\calC_1, \calC_2,...,\calC_l$ be the current clusters in nonincreasing order of size.  
We find the minimum index $j \in [1,l]$ such that there exists a vertex $v$ not yet clustered, with the highest average membership to $\calC_j$, that is {\sf  Membership}($v, \cC_j$) $ \geq $ {\sf  Membership}($v, \cC_{j'}$), $\forall j' \neq j$, and $j$ is the smallest index for which such a $v$ exists. We first check if $v \in \calC_j$ by querying $v$ with any current member of $\calC_j$. If not, then we group the clusters $\calC_1,\calC_2,..,\calC_{j-1}$ in at most $\lceil \log{n} \rceil$ groups such that clusters in group $i$ has size in the range $[\frac{|\calC_1|}{2^{i-1}}, \frac{|\calC_1|}{2^i})$. For each group, we pick the cluster which has the highest average membership with respect to $v$, and check by querying whether $v$ belongs to that cluster. Even after this, if the membership of $v$ is not resolved, then we query $v$ with one member of each of the clusters that we have not checked with previously. If $v$ is still not clustered, then we create a new singleton cluster with $v$ as its sole member.

The pseudocode of the algorithm is given in Figure~\ref{algo:cc-exact}
\begin{figure}
\caption{Pseudocode: Las Vegas Algorithm  for~\cc}
\label{algo:cc-exact}
\includegraphics[width=6in]{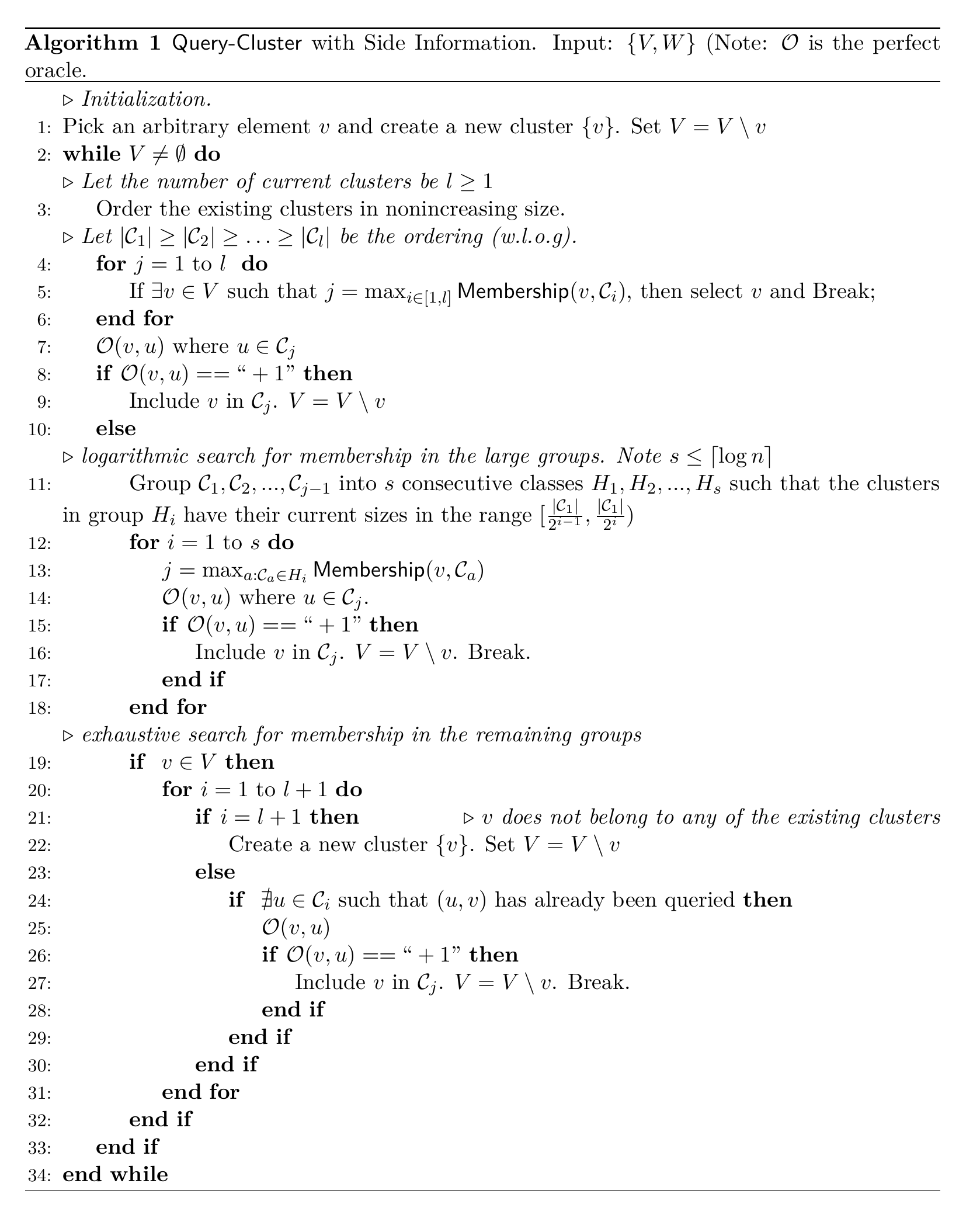}
\end{figure}
We now give a proof of the Las Vegas part of Theorem  \ref{thm:div} here using Algorithm \ref{algo:cc-exact}. We crucially use the following lemma which proves a strong concentration inequality adapting the Sanov's Theorem (see Lemma \ref{thm:sanov}) of information theory.

 \begin{lemma}\label{lem:unknown}
Suppose, $\cC, \cC' \subseteq V$, $\cC \cap \cC' = \emptyset$ and $|\cC| \ge M, |\cC'| \ge M =\frac{32 \log n}{ \cH^2(f_+\|f_-)}$. Then, 
$$
\Pr\Big({\sf Membership}(v, \cC') \ge {\sf Membership}(v, \cC) \mid v \in \cC\Big) \le \frac{2}{n^3}.
$$
\end{lemma}

\begin{proof}%[Proof of Lemma~\ref{lem:unknown}]
Let $\beta = \frac{\cH(f_+\|f_-)}{2}$. If ${\sf Membership}(v, \cC') \ge {\sf Membership}(v, \cC)$ then we must have,
$\cH(p_{v,\cC'} \| p_{\cC'}) \le \cH(p_{v,\cC} \| p_{\cC}) .$ This means, either
$
\cH(p_{v,\cC'} \| p_{\cC'})  \le \frac{\beta}{2}
$
or $ \cH(p_{v,\cC} \| p_{\cC})  \ge \frac{\beta}{2}.$ Now, using triangle inequality, 
\begin{align*}
&\Pr\Big(\cH(p_{v,\cC'} \| p_{\cC'})   \le \frac{\beta}{2} \Big) \le \Pr\Big(\cH(p_{v,\cC'} \| f_+)  - \cH(p_{\cC'} \| f_+) \le \frac{\beta}2\Big)\\
&  \le \Pr\Big(\cH(p_{v,\cC'} \| f_+)  \le \beta \text{ or }\cH(p_{\cC'} \| f_+)  \ge \frac{\beta}2\Big)
\le  \Pr\Big(\cH(p_{v,\cC'} \| f_+)  \le \beta\Big) + \Pr\Big(\cH(p_{\cC'} \| f_+)  \ge \frac{\beta}2\Big).
\end{align*}
Similarly,
\begin{align*}
&\Pr\Big(\cH(p_{v,\cC} \| p_{\cC})  \ge \frac{\beta}{2} \Big) \le \Pr\Big(\cH(p_{v,\cC} \| f_+)  + \cH(p_{\cC} \| f_+)  \ge \frac{\beta}2\Big)\\
&  \le \Pr\Big(\cH(p_{v,\cC} \| f_+)   \ge \frac{\beta}4 \text{ or }  \cH(p_{\cC} \| f_+)  \ge \frac{\beta}4\Big)
\le  \Pr\Big(\cH(p_{v,\cC} \| f_+)   \ge \frac{\beta}4\Big) + \Pr\Big( \cH(p_{\cC} \| f_+)  \ge \frac{\beta}4\Big).
\end{align*}
Now, using Sanov's theorem (Lemma \ref{thm:sanov}), we have,
$$
 \Pr\Big(\cH(p_{v,\cC'} \| f_+)  \le \beta\Big) \le (M+1)^q \exp(-M \underset{p: \cH(p \| f_+)  \le \beta}\min D(p\|f_-)).
$$
At the optimizing $p$ of the exponent, 
\begin{align*}
D(p\|f_-)& \ge 2\cH^2(p\|f_-) & \text{relation between Hellinger and KL  \cite{sason2016f}} \\
& \ge 2(\cH(f_+\|f_-) - \cH(p\|f_+))^2 & \text{ from using triangle inequality} \\
& \ge 2(2\beta - \beta)^2 & \text{ from noting the value of $\beta$}\\ 
&= \frac{\cH^2(f_+\|f_-)}{2}. %\\
\end{align*}

Again, using Sanov's theorem (Lemma \ref{thm:sanov}), we have,
$$
\Pr\Big(\cH(p_{\cC'} \| f_+)  \ge \frac{\beta}2\Big) \le (M+1)^q \exp(-M \underset{p: \cH(p \| f_+)  \ge \frac{\beta}{2}}\min D(p\|f_+)).
$$
At the optimizing $p$ of the exponent, 
\begin{align*}
D(p\|f_+)& \ge 2\cH^2(p \|f_+) & \text{relation between Hellinger and KL divergences \cite{sason2016f}} \\
& \ge \frac{\beta^2}{2} & \text{ from noting the value of $\beta$}\\ 
&= \frac{\cH^2(f_+\|f_-)}{8}.\\
\end{align*}

Now substituting this in the exponent,
 using the value of $M$ and doing the same exercise for the other two probabilities we get the claim of the lemma.
\end{proof}

\begin{proof}[Proof of Theorem \ref{thm:div}, Las Vegas Algorithm] First, The algorithm never includes a vertex in a cluster without querying it with at least one member of that cluster.  Therefore, the clusters constructed by our algorithm are always proper subsets of the original clusters.  Moreover, the algorithm never creates a new cluster with a vertex $v$ before first querying it with all the existing clusters. Hence, it is not possible that two clusters produced by our algorithm can be merged.

Let $\calC_1,\calC_2,...,\calC_l$ be the current non-empty clusters that are formed by Algorithm \ref{algo:cc-exact}, for some $l \leq k$. Note that Algorithm \ref{algo:cc-exact} does not know $k$. Let without loss of generality $|\calC_1| \geq |\calC_2| \geq ...\geq |\calC_l|$. Let there exists an index $i \leq l$ such that $|\calC_1| \geq |\calC_2| \ge \dots\geq |\calC_{i}| \geq M$,  where 
$M = \frac{32 \log n}{ \cH^2(f_+\|f_-)}$. %, where $\min_i f_+(i), \min_i f_-(i) \ge \epsilon$ for a constant $\epsilon$.
%$M = \frac{6\log n}{\theta_{gap}^2}$. 
Of course, the algorithm does not know either $i$ or $M$. If even $|\calC_1| < M$, then $i=0$. Suppose $j'$ is the minimum index such that there exists a vertex $v$ with highest average membership in $\calC_{j'}$. There are few cases to consider based on $j' \leq i$, or $j' > i$ and the cluster  that truly contains $v$.

{\it Case 1. $v$ truly belongs to $\calC_{j'}$.} In that case, we just make one query between $v$ and an existing member of $\calC_{j'}$ and the first query is successful.

{\it Case 2. $j' \leq i$ and $v$ belongs to $\calC_j, j \neq j'$ for some $j \in \{1,\dots, i\}$.} 
%Let $\average(v,\calC_{j})$ and $\average(v,\calC_{j'})$ be the average membership of $v$ to $\calC_j$, and $\calC_{j'}$ respectively. Then we have $\average(v,\calC_{j'}) \geq \average(v,\calC_{j})$, that is
 Here we have {\sf  Membership}($v, \cC_{j'}$) $ \geq $ {\sf  Membership}($v, \cC_{j}$). Since both $\calC_j$ and $\calC_{j'}$ have at least $M$ current members, then using
 Lemma \ref{lem:unknown}, this happens with probability at most $\frac{2}{n^3}$. Therefore, the 
  number of queries involving $v$ before its membership gets determined is $\leq 1$ with probability at least $1- \frac{2k}{n^3}$.
 %expected number of queries involving $v$ before its membership gets determined is $\leq 1+\frac{2k}{n^3} < 2$.
% This is only possible if either $\average(v,\calC_{j'}) \geq \mu_r+\frac{\theta_{gap}}{2}$ or $\average(v,\calC_{j}) \leq \mu_g-\frac{\theta_{gap}}{2}$. Since both $\calC_j$ and $\calC_j'$ have at least $M$ current members, then using the Chernoff-Hoeffding's bound (Lemma \ref{lem:hoef1}) followed by union bound this happens with probability at most $\frac{2}{n^3}$. Therefore, the expected number of queries involving $v$ before its membership gets determined is $\leq 1+\frac{2}{n^3} k < 2$.

{\it Case 4. $v$ belongs to $\calC_j, j \neq j'$ for some $j > i$.} In this case the algorithm may make $k$ queries involving $v$ before its membership gets determined.

{\it Case 5. $j' > i$, and $v$ belongs to $\calC_j$ for some $j \leq i$.} 
In this case, there exists no $v$ with its highest membership in $\calC_1,\calC_2,...,\calC_i$.

Suppose $\calC_1,\calC_2,...,\calC_{j'}$ are contained in groups $H_1,H_2,...,H_s$ where $s \leq \lceil \log{n}\rceil$. Let $\calC_j \in H_t$, $t \in [1,s]$. Therefore, $|\calC_j| \in [\frac{|\calC_1|}{2^{t-1}}, \frac{|\calC_1|}{2^{t}}]$. If $|\calC_j| \geq 2M$, then all the clusters in group $H_t$ have size at least $M$. 
Now with probability at least $1-\frac{2}{n^2}$,  {\sf  Membership}($v, \cC_j$) $ \geq $ {\sf  Membership}($v, \cC_{j''}$) for every cluster $\calC_{j''} \in H_t$. In that case, the membership of $v$ is determined within at most $\lceil \log{n}\rceil$ queries. Else, with probability at most $\frac{2}{n^2}$, there may be $k$ queries to determine the membership of $v$.

%Now with probability at least $1-\frac{2}{n^2}$, $\average(v,\calC_j) \geq \average(v,\calC_{j''})$, that is {\sf  Membership}($v, \cC_j$)$ \geq ${\sf  Membership}($v, \cC_{j''}$) for every cluster $\calC_{j''} \in H_t$. In that case, the membership of $v$ is determined within at most $\lceil \log{n}\rceil$ queries. Otherwise, with probability at most $\frac{2}{n^2}$, there may be $k$ queries to determine the membership of $v$.

Therefore, once a cluster has grown to size $2M$, the number of queries to resolve the membership of any vertex in those clusters is at most $\lceil \log{n} \rceil$ with probability at least $1-\frac{2}{n^2}$. Hence, for at most $2kM$ elements, the number of queries made to resolve their membership can be $k$. Thus the total number of queries made by Algorithm \ref{algo:cc-exact} is 
$O(n\log{n}+Mk^2)=O(n\log{n}+\frac{k^2\log{n}}{ \cH^2(f_+\|f_-)})$ with probability $1-o_n(1)$. %$O(n\log{n}+Mk^2)=O(n\log{n}+\frac{k^2\log{n}}{(\mu_+-\mu_{-})^2})$. Moreover, if we knew $\mu_{+}$ and $\mu_{-}$, we can calculate $M$, and thus whenever a clusters grows to size $M$, remaining of its members can be included in that cluster without making any error with high probability. This leads to Theorem \ref{thm:mu}.
\end{proof}

\begin{remark}
While, for the more general setting with unknown $f_{i,j}$s (distribution referring to similarity of cluster $i$ and $j$), we do not know how to extend this algorithm yet,
 if the parameters were known it is possible to extend our algorithm to such setting. We can calculate $M_i=O(\frac{\log{n}}{ \min_{j: j \neq i}\cH^2(f_{i,i}\|f_{i,j})}),$ and thus whenever the $i$ th clusters grows to size $M_i$, remainder of its members can be inferred.  
\end{remark}
Since, we handle very generic distributions, our upper bounds are off by a factor of $O(\log{n})$ from the lower bound. Tightening this bound, e.g. for sparse SBM to match the conjectured trade-off between queries and threshold remains an important open question.

~\\
\paragraph*{\bf Discussion.} This is the first rigorous theoretical study of interactive clustering with side information, and it unveils many interesting directions for future study of both theoretical and practical significance (see Appendix~\ref{sec:future} for more details). Having arbitrary $f_+$, $f_-$ significantly generalizes SBM. Also it raises an important question about how SBM recovery threshold changes with queries. For sparse region of SBM, where $f_+$ is ${\rm Bernoulli}(\frac{a'\log{n}}{n})$ and $f_-$ is ${\rm Bernoulli}(\frac{b'\log{n}}{n})$, $a' > b'$, Lemma \ref{lemma:funda} is not tight yet. However, it shows the following trend.  By setting $a=\frac{n}{k}$ and ignoring the lower order terms and a $\sqrt{\log{n}}$ factor, recovery error becomes $\approx (1-\frac{Q}{nk})-\frac{\cH(a'\|b')}{\sqrt{k}}$. We conjecture with $Q$ queries, the sharp recovery threshold of sparse SBM changes from $(\sqrt{a'}-\sqrt{b'}) \geq  \sqrt{k}$ to $(\sqrt{a'}-\sqrt{b'}) \geq  \sqrt{k}\left(1-\frac{Q}{nk}\right)$. Proving this bound remains an exciting open question.
%For SBM with $f_+\sim {\rm Bernoulli}(\frac{a\log{n}}{n})$ and $f_-\sim {\rm Bernoulli}(\frac{b\log{n}}{n})$, Lemma \ref{lemma:funda} has a gap of $\sim \sqrt{\log{n}}$. However, it reveals the following trend. Note that we have $\cH(f_+\|f_-)=\frac{1}{2}\sqrt{\frac{\log{n}}{n}}|\sqrt{a}-\sqrt{b}|$. Hence ignoring the lower order terms and a $\sqrt{\log{n}}$ factor, we see the probability of error will be nonzero if $|\sqrt{a}-\sqrt{b}| < \sqrt{k}\left(1-\frac{Q}{nk}\right)$ since each cluster has a size of $\frac{n}{k}$. Characterizing this gap exactly remains an exciting future direction.
%To conclude, we also refer the readers to Appendix~\ref{sec:future} for future direction. 
 
% Their upper and lower bounds on number of queries are both  weak compared to our results. For example, they only show a lower bound of $\Omega(\log{k}+\log{n})$. We believe using our method, this bound can be substantially improved, which would be an interesting future work. 

% \newpage 
\bibliographystyle{abbrv}

{%\small
\bibliography{bibfile}
}

%\newpage
  
\appendix

\section{Zero Query and the Stochastic Block Model}\label{sec:zero}
\label{sec:SBM-general}

Consider the case when we allow zero query to the oracle. The clustering has to be done just by using the side information matrix. 
This is a direct generalization to the well-known stochastic block model. Indeed, if $f_+$ is Bernoulli($p$) and  $f_- $ is Bernoulli($q$), 
then the side information matrix is a binary matrix, as in the case of stochastic block model \cite{abh:16,hajek2016achieving,chin2015stochastic,mossel2015consistency}.

It is clear that if the clustering input instance is adversarial, then it is impossible to recover the clusters with high probability. For example, think of the situation that 
$k-1$ clusters are of size $1$ each. In that case, one of these $k-1$ small cluster points cannot be assigned to the correct cluster without querying, with a positive probability.
Note that, we will not be able to have such an argument later when querying is allowed, which makes that case significantly difficult.

Let us look at the scenario, when there are $k$ clusters of size $\frac{n}{k}$ each. Suppose $V = \sqcup_{i=1}^k V_i$ is the correct clustering. Consider the 
different clustering instances, that can be derived from the correct clustering, by swapping any two points $a\in V_i$ and $b\in V_j$, $i \ne j$. There are $K = \binom{k}{2}\frac{n^2}{k^2} = \frac{n^2}{2}(1-1/k)$ such different clusterings (partitions) possible. Let us consider these $K$ different cases as $K$ hypotheses, and try to identify which one of them is true based on the side information matrix.

Let $Q_t, t =1, \dots, K$ be the joint probability distributions of the side information matrix under hypothesis $t, t =1, \dots , K$. Also, let the correct clustering be the zeroth hypothesis and induces a joint probability distribution $Q_0$. 

In this type of multi-hypothesis testing problem, a standard tool to lower bound probability of error is Fano's inequality. However, Fano's inequality in its usual form in hypothesis testing (see,  \cite[Thm.~7]{han1994generalizing}) does not give the tightest possible result in our case. We instead use another form of Fano's inequality 
 from \cite[Thm.~II.1 Eq. (5)]{guntuboyina2011lower} - therein taking $Q= Q_0$ and taking $f(x) = x\log x$, we have,
  the probability of error $ P_e$ of
this hypothesis testing problem (to identify between the $K$ hypotheses) given by,
$$
\frac{1}{K}\sum_i D(Q_i \| Q_0) \ge (1-P_e) \log (K(1-P_e)) + P_e \log (KP_e/ (K-1))
$$
 where $D(f\| g)$ is the Kullback-Leibler (KL) divergence. The KL divergence between joint distribution of independent random variables is sum of the KL divergence of the marginals, and the only times when the distributions of $w_{i,j}$ differs under $Q_i$ and under $Q_0$ is when $i$ or $j$ belong to the two clusters where elements were swapped. There are $\frac{4n}{k}$ such instances, among them $\frac{2n}k$ contributes $D(f_+\|f_-)$ to the sum and $\frac{2n}k$ contributes $D(f_-\|f_+)$ to the sum.
 Therefore we obtain,
\begin{align*}
P_e &\ge 1- \frac{\frac{1}{K}\sum_i D(Q_i \| Q_0)+ \log 2}{ \log K}
 \ge 1 - \frac{\frac{2n}{k}\Delta(f_+ \| f_1)}{\log \frac{n^2}{2}(1-1/k) } 
 \approx 1 - \frac{n \Delta(f_+ \| f_1) }{k \log n},
\end{align*}
where $\Delta(f\|g) \equiv D(f\|g)+D(f\|g)$ .

%Now using a special version of Fano's inequality, the probability of error $P_e$ of this hypothesis testing  problem is given by (see details in the Appendix), %  \ref{appendix:zero}),
%\begin{align*}%\label{eq:fano}
%P_e & \ge 1 - \frac{n \Delta(f_+ \| f_1) }{k \log n},
%\end{align*}
%where $\Delta(f\|g) \equiv D(f\|g)+D(f\|g)$ and $D(f\| g)$ is the Kullback-Leibler (KL) divergence.

%The second inequality is true because
%the KL divergence between joint distribution of independent random variables is sum of the KL divergence of the marginals, and the only times when $w_{i,j}$ differs under $Q_i$ and under $Q_0$ is when $i$ or $j$ belong to the two clusters where elements were swapped.

%\begin{align}
%P_e &\ge 1- \frac{\sum_i \sum _j D(Q_i \| Q_j)+ \log 2}{K^2 \log K}\\
%& \ge 1 - \frac{\frac{4n}{k}\Delta(f_+ \| f_1)}{\log \frac{n^2}{2}(1-1/k) } \\
%& \approx 1 - \frac{2n \Delta(f_+ \| f_1) }{k \log n}.
%\end{align}

One particular regime of interest in the literature of stochastic block model appear (see,  \cite{DBLP:conf/focs/AbbeS15,mossel2015consistency}) when, $f_+ \sim \text{Bernoulli}\Big(\frac{a \log n}{n}\Big)$ and $f_- \sim \text{Bernoulli}\Big(\frac{b \log n}{n}\Big)$.
Then 
$
D(f_+\|f_-) = \frac{a \log n}{n} \log \frac{a}{b} + \Big(1-  \frac{a \log n}{n} \Big) \log \frac{1-  \frac{a \log n}{n}  }{1-  \frac{b \log n}{n} }
$
and
$
\Delta(f_+\| f_-) = (a-b) \frac{\log n}{n} \Big(\log\frac{a}{b} - \log \frac{1-  \frac{a \log n}{n}  }{1-  \frac{b \log n}{n} } \Big) \approx  \frac{\log n}{n} \cdot (a-b)\log \frac{a}{b}.
$
In this case, 
$
P_e \ge 1 - \frac{a-b}{k}\log \frac{a}{b},
$
and, $P_e >0$ as long as $(a-b)\log \frac{a}{b} <k$. This lower bound can be improved by considering generalized versions of Fano's inequality
involving Hellinger divergence.

In particular, by constructing a different hypothesis testing scenario and using  a generalized version of Fano's inequality we can obtain  the following result on probability $P_e$ of erroneous clustering.
In particular, we can use a generalized version of Fano's inequality due to Polyanskiy and V\'erdu \cite[Thm.~4]{polyanskiy2010arimoto}.
%Consider the case when each element has equal probability $\frac1k$ to be from any of the $k$ clusters. Suppose, we have the task of assigning 
%an element to one of the clusters based on the side information matrix. Clearly this is a $k$-hypothesis problem. Using the result of \cite[Thm.~4]{polyanskiy2010arimoto}, 
%we can lower bound the probability of error of this hypothesis testing problem. Indeed,
Consider the following different hypothesis testing situation. Suppose $k$ divides $n$, and there are $k$ equally sized subsets that partition the set of elements $[n] = V_1 \sqcup V_2 \sqcup \dots \sqcup V_k$. Let $v\in V_1$ be a fixed element. Take any cluster $V_j, j \ne 1$. For all elements $u_1, \dots, u_{n/k} \in V_j$, we obtain   $K = n/k$ different hypotheses by interchanging $v$ with $u_i, i =1, \dots, n/k$.  We consider the probability of error of this hypothesis testing problem. 
Inparticular,  \cite[Thm.~4]{polyanskiy2010arimoto}, says that the probability of error $P_e$ is given by (considering Renyi divergence of order $\frac12$),  
\begin{align*}
- 2\log \Big(\sqrt{\frac{1-P_e}K} + \sqrt{P_e(1-\frac1K)}\Big)  \le - \log \sum_y (\frac{1}{K} \sum_{j=1}^K \sqrt{Q_j(y)})^2
\end{align*}
which implies for us,
\begin{align*}
 \Big(&\sqrt{\frac{1-P_e}K} + \sqrt{P_e(1-\frac1K)}\Big)^2  \ge \frac{1}{K^2}  \sum_{j}\sum_i \sum_y  \sqrt{Q_j(y)Q_i(y)}
  =  \frac{1}{K^2}  \sum_{j}\sum_i \Big(1- \cH^2(Q_i \| Q_j)\Big)\\
 & = 1 - \cH^2(Q_i \| Q_j) = 1 -\Big(1-(1-\cH^2(f_+\|f_-))^{\frac{4n}{k}}\Big) = (1-\cH^2(f_+\|f_-))^{\frac{4n}{k}},
 \end{align*}
 where we had  to crucially used the following fact: if $P_1^m$  and $Q_1^m$ denote  joint distributions of $m$ of independent  $P_i$ and independent $Q_i, i =1, \dots, m$ random variables, then,
 \begin{align*}
 \cH^2(P_1^m \| Q_1^m) &= 1- \int_{x_1, \dots, x_m} \sqrt{P_1^m(x_1, \dots, x_m)Q_1^m(x_1, \dots, x_m)} dx_1, \dots, dx_m\\
 & = 1- \prod_{i=1}^m \int_{x}\sqrt{P_i(x)Q_i(x)} dx & \text{using Tonelli's theorem}\\
 & = 1- \prod_{i=1}^m(1-\cH^2(P_i\|Q_i))  \le \sum_{i=1}^m \cH^2(P_i\|Q_i).
 \end{align*}
 
  Again, we assume $f_+ \sim \text{Bernoulli}\Big(\frac{a \log n}{n}\Big)$ and $f_- \sim \text{Bernoulli}\Big(\frac{b \log n}{n}\Big)$. In this case,
 \begin{align*} 
 \sqrt{\frac{k}{n}} +\sqrt{P_e} & \ge \Big(\sqrt{ab}\frac{\log n}{n} + \sqrt{(1-\frac{a\log n}{n})(1-\frac{b \log n}{n})}\Big)^{\frac{2n}{k}}= \Big(1- \Big(\frac{a+b}{2}-\sqrt{ab}- \frac{ab \log n}{n}\Big)\frac{\log n}{n}  \Big)^{\frac{2n}{k}}\\
 & \approx e^{- \Big(\frac{a+b}{2}-\sqrt{ab}- \frac{ab \log n}{n}\Big)\frac{2\log n}{k}} = n^{-\Big(\frac{a+b}{2}-\sqrt{ab}- \frac{ab \log n}{n}\Big)\frac{2}{k}}.
\end{align*}

This implies,
$
\sqrt{P_e} \ge n^{-\Big(\frac{a+b}{2}-\sqrt{ab}\Big)\frac{2}{k}} - \sqrt{k} n^{-1/2}
$
In particular, if 
$
\Big(\frac{a+b}{2}-\sqrt{ab}\Big)\frac{2}{k} <\frac12,
$
then $P_e > 0$. Hence,  $P_e > 0$ if 
$$
\sqrt{a} - \sqrt{b} <\sqrt{\frac{k}2}.
$$

 While in this regime, this result is slightly suboptimal compared to the lower bound of \cite{DBLP:conf/focs/AbbeS15}, where the corresponding bound was $\sqrt{a} - \sqrt{b} <\sqrt{k}$, note that our bound works for arbitrary $f_+, f_-$ and across all regimes; moreover we have not tried to optimize the constants here.

%\input{simple-lower-bound}

%  $P_e > \frac{1}{n}$ if 
%$$
%\sqrt{a} - \sqrt{b} <\sqrt{\frac{k}2}.
%$$

%Let us  look at the first scenario. There are $k$ clusters of size $\frac{n}{k}$ each. Suppose $V = \sqcup_{i=1}^k V_i$ is the correct clustering.  The $K$
%different clustering instances can be derived from the correct clustering, by swapping any two points $a\in V_i$ and $b\in V_j$, $i \ne j$. Here $K = \binom{k}{2}\frac{n^2}{k^2} = \frac{n^2}{2}(1-1/k)$. We consider these $K$ different cases as $K$ hypotheses, and try to identify which one of them is true based on the side information matrix.
%
%Let $Q_t, t =1, \dots, K$ be the joint probability distributions of the side information matrix under hypothesis $t, t =1, \dots , K$. Also, let the correct clustering be the zeroth hypothesis and induces a joint probability distribution $Q_0$. 

%\vspace{0.1in}

%Using the generalized Fano's inequality and using some properties of the
%Hellinger distribution (see Appendix for details), we have,
%\begin{align}\label{eq:fano2}
% \Big(&\sqrt{\frac{1-P_e}K} + \sqrt{P_e(1-\frac1K)}\Big)^2  \ge (1-\cH^2(f_+\|f_-))^{\frac{4n}{k}}.
%\end{align}
\section{Connections \& Future Direction}
\label{sec:future}
This is the first work that rigorously study the query complexity of clustering with side information. We introduce new general information theoretic methods; as well as use, information theoretic inequalities to design efficient algorithms for clustering with near-optimal complexity. Our algorithms are entirely parameter free, and are computationally efficient. This work reveals interesting connection to the well-studied model of the stochastic block model and, generalize them in a significant way by considering arbitrary distribution for noise opposed to only Bernoulli noise, and opens up new direction of study in the general area of clustering and community detection.

Even for the zero-query case, using  generalized Fano's inequality in multiple hypothesis testing, we can derive simple lower bounds for SBM with arbitrary $f_+, f_-$ and cluster size distribution, matching closely the bounds for the sparse region $f_+\sim \text{Bernoulli}(\frac{a\log{n}}{n})$ and $f_-\sim \text{Bernoulli}(\frac{b\log{n}}{n})$ and cluster size $\sim \frac{n}{k}$. Extending this lower bound to consider adaptive querying comes as a major challenge, as querying may reveal different deterministic information under different hypothesis. We propose a general framework for deriving such lower bounds, and in the process it reveals an interesting trend on how the threshold of recovery should change with querying: from $\sqrt{a}-\sqrt{b} \geq \sqrt{k}$ to $\sqrt{a}-\sqrt{b} \geq \sqrt{k}\left(1-\frac{Q}{nk}\right)$ (see Lemma~\ref{lemma:funda}). That is querying can help reduce the threshold when $O(n)$ edges have been queried as $k$ is a constant. Currently, there is a $\sqrt{\log{n}}$ gap to achieve this bound as our lower bounds deal with very generic distributions and cluster sizes.  Closing this gap for the stochastic block model with querying remains an interesting open question.

 We propose two computationally efficient algorithms that match the query complexity lower bound within $\log{n}$ factor and are completely parameter free. In particular, our iterative-update method to design Monte-Carlo algorithm provides a general recipe to develop any parameter-free algorithm, which are of extreme practical importance. The convergence result is established by extending Sanov's theorem from the large deviation theory which gives bound only in terms of KL-divergence. 
  Due to the generality of the distributions, the only tool we could use is Sanov's theorem. 
 However, Hellinger distance comes out to be the right measure both for lower and upper bounds. If $f_+$ and $f_-$ are common distributions like Gaussian, Bernoulli etc., then other concentration results stronger than Sanov may be applied to improve the constants and a logarithm factor to show the trade-off between queries and thresholds as in sparse SBM. While some of our results apply to general $f_{i,j}$s, a full picture with arbitrary $f_{i,j}$s and closing the gap of $\log{n}$ between the lower and upper bound remain an important future direction.
 
 There is also a very recent result by \cite{ashtiani2016clustering} that studies the specific $k$-means clustering problem with a different side information model. 
 While the setting is quite different, we believe their results can be significantly improved (for example, they only show a lower bound of $\Omega(\log{k}+\log{n})$ to overcome NP-hardness of the problem) using our general methods - which promises to be an interesting future work.
\end{document}